\documentclass[sn-mathphys,Numbered]{sn-jnl}

\usepackage{graphicx}%
\usepackage{multirow}%
\usepackage{amsmath,amssymb,amsfonts}%
\usepackage{amsthm}%
\usepackage[title]{appendix}%
\usepackage{xcolor}%
\usepackage{textcomp}%
\usepackage{manyfoot}%
\usepackage{booktabs}%
\usepackage{algorithm}%
\usepackage{lmodern}%
\usepackage{algorithmicx}%
\usepackage{algpseudocode}%
\usepackage{listings}%
\usepackage{subfig}%
\usepackage{parskip}%
\usepackage{pgfplots}%
\usepgfplotslibrary{statistics}%
\usetikzlibrary{calc}
\usepackage{xspace}
\pgfplotsset{compat=newest}
\pgfplotsset{
/pgfplots/custom legend/.style={
legend image code/.code={
\draw [only marks,mark=diamond*]
plot coordinates { 
(0.3cm,0cm)
};
}, },}

\newcommand{\bbl}{\begin{block}}
\newcommand{\ebl}{\end{block}}
\newcommand{\bc}{\begin{center}}

\newcommand{\bea}{{\begin{align}}}
\newcommand{\beq}{\begin{equation}}
\newcommand{\beqa}{\begin{eqnarray}}
\newcommand{\bi}{\begin{itemize}}
\newcommand{\bitem}{\begin{itemize}}
\newcommand{\bpm}{\begin{pmatrix}}

\newcommand{\cref}[1]{ {\tiny[{#1}]}}

\newcommand{\DD}{\mathcal{D}}

\newcommand{\ec}{\end{center}}
\newcommand{\eea}{{\end{align}}}
\newcommand{\eeq}{\end{equation}}
\algdef{SE}[REPEATN]{REPEATN}{END}[1]{\algorithmicrepeat\ #1 \textbf{times}}{\algorithmicend}
\newcommand{\eeqa}{\end{eqnarray}}
\newcommand{\ei}{\end{itemize}}
\newcommand{\euler}{\mathrm{eu}}
\newcommand{\expon}{\mathrm{ex}}
\newcommand{\id}{{id}}
\newcommand{\rprod}{\coprod}
\newcommand{\eitem}{\end{itemize}}

\newcommand{\epm}{\end{pmatrix}}

\newcommand{\fg}{\mathfrak{g}}

\newcommand{\N}{\mathbb{N}}

\newcommand{\R}{\mathbb{R}}

\newcommand{\Rn}{\mathbb{R}^n}

\newcommand{\seq}{\subseteq}

\newcommand{\mypar}[1]{\medskip\noindent\textbf{#1.} }

\newcommand{\tm}[1]{}

\newcommand{\half}{\frac{1}{2}}

\newcommand{\LL}{{\mathcal{L}}}

\newcommand{\bpf}{\begin{proof}}
\newcommand{\epf}{\end{proof}}

\newcommand{\be}{\begin{enumerate}}
\newcommand{\ee}{\end{enumerate}}

\newcommand{\barx}{\Bar{x}}
\newcommand{\frobenius}{F}

\DeclareMathOperator{\SO}{SO}
\DeclareMathOperator{\Tgp}{T}
\DeclareMathOperator{\SE}{SE}
\DeclareMathOperator{\SIM}{SIM}
\DeclareMathOperator{\Aff}{Aff}
\DeclareMathOperator{\GL}{GL}
\DeclareMathOperator{\PGL}{PGL}

\DeclareMathOperator{\SSIM}{SSIM}
\DeclareMathOperator{\RMSE}{RMSE}
\DeclareMathOperator{\NCC}{NCC}
\DeclareMathOperator{\DS}{DS}

\usepackage{soul}

\newcommand{\banach}{{\mathcal{B}}}

\definecolor{oceangreen}{RGB}{0,57,74}
\definecolor{oceangreen-90}{RGB}{0,73,90}
\definecolor{oceangreen-80}{RGB}{0,90,107}
\definecolor{oceangreen-70}{RGB}{17,107,123}
\definecolor{oceangreen-60}{RGB}{60,126,142}
\definecolor{oceangreen-50}{RGB}{93,146,160}
\definecolor{oceangreen-40}{RGB}{126,168,178}
\definecolor{oceangreen-30}{RGB}{157,188,198}
\definecolor{oceangreen-20}{RGB}{189,210,216}
\definecolor{oceangreen-10}{RGB}{222,233,236}
\definecolor{red-lightest}{RGB}{233,191,173}
\definecolor{red-light}{RGB}{227,32,50}
\definecolor{red-medium}{RGB}{182,22,32}
\definecolor{red-dark}{RGB}{126,20,24}
\definecolor{orange-lightest}{RGB}{240,205,178}
\definecolor{orange-light}{RGB}{236,117,4}
\definecolor{orange-medium}{RGB}{202,81,25}
\definecolor{orange-dark}{RGB}{129,53,18}
\definecolor{yellow-lightest}{RGB}{236,217,186}
\definecolor{yellow-light}{RGB}{251,186,0}
\definecolor{yellow-medium}{RGB}{191,134,21}
\definecolor{yellow-dark}{RGB}{117,83,17}
\definecolor{green-lightest}{RGB}{219,215,187}
\definecolor{green-light}{RGB}{149,187,12}
\definecolor{green-medium}{RGB}{126,133,37}
\definecolor{green-dark}{RGB}{50,89,74}
\definecolor{turquoise-lightest}{RGB}{205,219,216}
\definecolor{turquoise-light}{RGB}{59,178,160}
\definecolor{turquoise-medium}{RGB}{36,143,133}
\definecolor{turquoise-dark}{RGB}{0,90,91}
\definecolor{ocean-lightest}{RGB}{198,220,226}
\definecolor{ocean-light}{RGB}{0,174,198}
\definecolor{ocean-medium}{RGB}{0,145,167}
\definecolor{ocean-dark}{RGB}{0,97,122}
\definecolor{cyan-lightest}{RGB}{195,217,229}
\definecolor{cyan-light}{RGB}{60,169,213}
\definecolor{cyan-medium}{RGB}{0,131,173}
\definecolor{cyan-dark}{RGB}{0,87,119}
\definecolor{blue-lightest}{RGB}{195,217,236}
\definecolor{blue-light}{RGB}{111,165,206}
\definecolor{blue-medium}{RGB}{0,105,163}
\definecolor{blue-dark}{RGB}{0,70,114}
\definecolor{gray-cold-lightest}{RGB}{215,223,228}
\definecolor{gray-cold-light}{RGB}{191,203,213}
\definecolor{gray-cold-medium}{RGB}{140,157,171}
\definecolor{gray-cold-dark}{RGB}{87,105,120}
\definecolor{gray-neutral-lightest}{RGB}{226,228,228}
\definecolor{gray-neutral-light}{RGB}{208,208,210}
\definecolor{gray-neutral-medium}{RGB}{156,157,160}
\definecolor{gray-neutral-dark}{RGB}{100,100,102}
\definecolor{gray-warm-lightest}{RGB}{237,236,227}
\definecolor{gray-warm-light}{RGB}{218,214,203}
\definecolor{gray-warm-medium}{RGB}{162,159,145}
\definecolor{gray-warm-dark}{RGB}{118,115,105}
\definecolor{blue-icy}{RGB}{103,148,150}

\newcommand{\firstcol}[1]{\begin{columns}[T]\begin{column}{#1}}

\newcommand{\lastcol}{\end{column}\end{columns}}

\usepackage{pgfplots}
\pgfplotsset{compat=1.13}
\usepgfplotslibrary{fillbetween}

\newlength{\logoheightinline}
\newcommand{\helm}[1]{}
\newcommand{\hamb}[1]{}

\newtheorem{theorem}{Theorem}
\begin{document}

\title[]{Stationary Velocity Fields on Matrix Groups for Deformable Image Registration}


\author*[1]{\fnm{Johannes} \sur{Bostelmann}}\email{johannes.bostelmann@uni-luebeck.de}

\author*[1]{\fnm{Ole} \sur{Gildemeister}}\email{o.gildemeister@uni-luebeck.de}

\author*[1]{\pfx{} \fnm{Jan} \sur{Lellmann}}\email{jan.lellmann@uni-luebeck.de}

\affil[1]{\orgdiv{Institute of Mathematics and Image Computing}, \orgname{University of Lübeck}, \orgaddress{\street{Maria-Goeppert-Str.~3}, \city{Lübeck}, \postcode{23562},  \country{Germany}}}

\abstract{The stationary velocity field (SVF) approach allows to build parametrizations of invertible deformation fields, which is often a desirable property in image registration. Its expressiveness is particularly attractive when used as a block following a machine learning-inspired network.
However, it can struggle with large deformations. We extend the SVF approach to matrix groups, in particular $\SE(3)$. This moves Euclidean transformations into the low-frequency part, towards which network architectures are often naturally biased, so that larger motions can be recovered more easily. This requires an extension of the flow equation, for which we provide sufficient conditions for existence. We further prove a decomposition condition that allows us to apply a scaling-and-squaring approach for efficient numerical integration of the flow equation. We numerically validate the approach on inter-patient registration of 3D MRI images of the human brain.
}




\maketitle
\section{Introduction}\label{sec1}
Image registration describes the task of aligning two or more images by providing a reasonable transformation. This task has a wide range of applications, including atlas-based segmentation \cite{wyawahare2009image}, tracking of changes in medical data such as tumor growth and fracture healing~\cite{sotiras2013deformable}, Synchronous Localization and Mapping~\cite{pfingsthorn2010maximum}, and image stitching for panoramic or (satellite/drone) captures \cite{sakharkar2013image}.

Typically, the sought transformation increases the image similarity while fulfilling additional criteria for a physically plausible deformation, such as its smoothness or invertibility. The computation of a useful alignment is often described by an optimization problem, in which an objective function is minimized over a specific space of parametrized deformations. 

Given two scalar-valued images represented by functions $I_1, I_2:\Omega\to\R$ on the image domain \mbox{$\Omega\seq\R^d$,} a typical approach for finding a suitable deformation $\phi:\Omega\to\R^d$ that maps corresponding points in $I_1$ to points in $I_2$ is to formulate a variational problem of the form
\begin{align}
    \min_{\phi \in \mathcal{D}} \LL(\phi;I_1,I_2),\quad  \LL(\phi; I_1,I_2):=J(I_1,I_2\circ \phi)+ R(\phi).\label{eq:reg-intro}
\end{align}
The set of deformations $\mathcal{D}$ is often chosen to consist of all functions $\phi:x\mapsto x+ u(x)$, where the displacement $u: \Omega \to \R^d$ is an element of a suitable topological vector space. The functional $J$ is a similarity term, in which smaller values indicate a higher similarity between its arguments, such as an $L^2$ distance; the regularizer $R:\mathcal{D}\to \mathbb{R}$ shall provide a bias towards physically reasonable deformations and can be used to stabilize the problem.

Choosing suitable similarity terms and regularizers has been intensively investigated~\cite{burger2013hyperelastic,wolterink2022implicit} and is not the goal of this work. Instead, we focus on another important aspect: the choice and parametrization of the deformation space $\DD$. In practice, expressiveness and convergence can depend heavily on the parametrization used. Moreover, this choice can influence the meaning and practicability of the regularizer.

\begin{figure}[tp]
    \centering
    \begin{tabular}{cccc}
      \quad\quad\rotatebox{90}{pre-aligned}&    \includegraphics[width=0.27\linewidth]{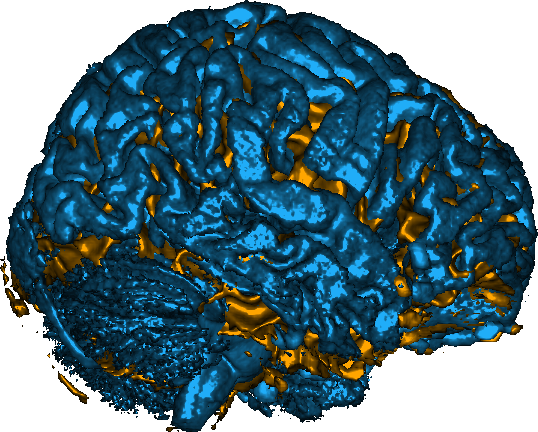}&
    \includegraphics[width=0.27\linewidth]{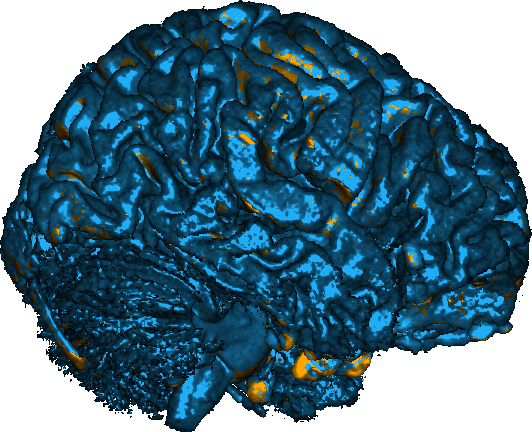}&
    \includegraphics[width=0.27\linewidth]{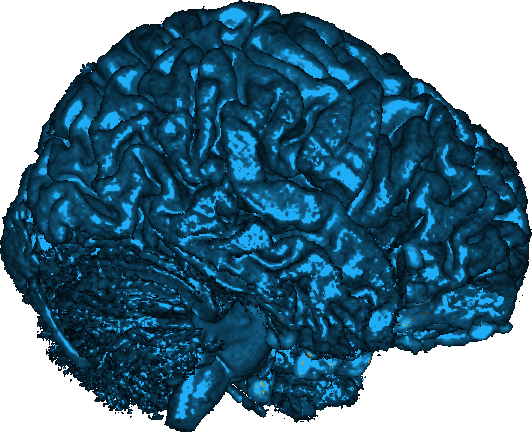}\\
        \quad\quad\rotatebox{90}{large deformation} & 
    \includegraphics[width=0.27\linewidth]{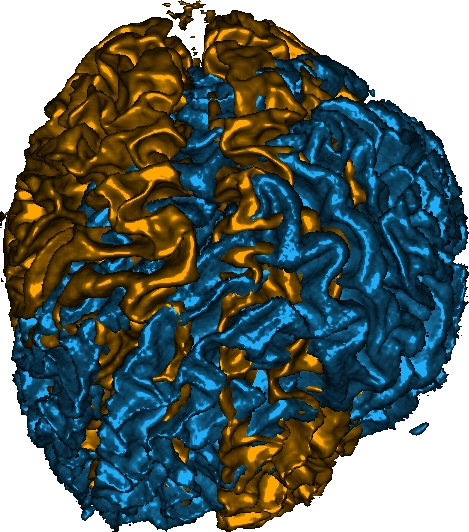} &
    \includegraphics[width=0.27\linewidth]{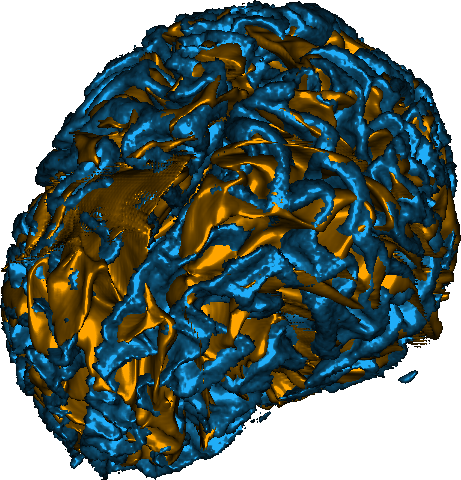} &
    \includegraphics[width=0.27\linewidth]{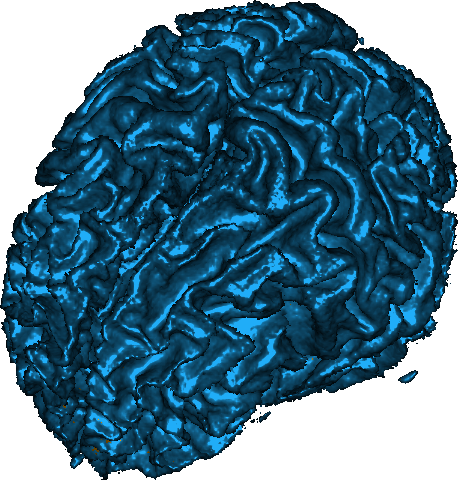}        
 \\
        & before registration & SVF & $\SE(3)$ (ours)
    \end{tabular}
        \vspace{0.5em}
    \caption{Comparison of the stationary velocity field (SVF) approach and our proposed matrix group-valued approach using the $\SE(3)$ group on synthetically deformed human brain MRI data. While both methods manage to align the 3D volumes under small deformations (top row), the SVF approach struggles when the input images are not pre-aligned (bottom row). The resulting deformation field shows clear alignment issues (center), which are alleviated by the proposed matrix group-valued approach (right).}
    \label{fig:alignments-iso}
\end{figure}
\begin{figure}[tp]
    \centering
    \subfloat[ground truth deformation]{
    \includegraphics[width=0.31\linewidth]{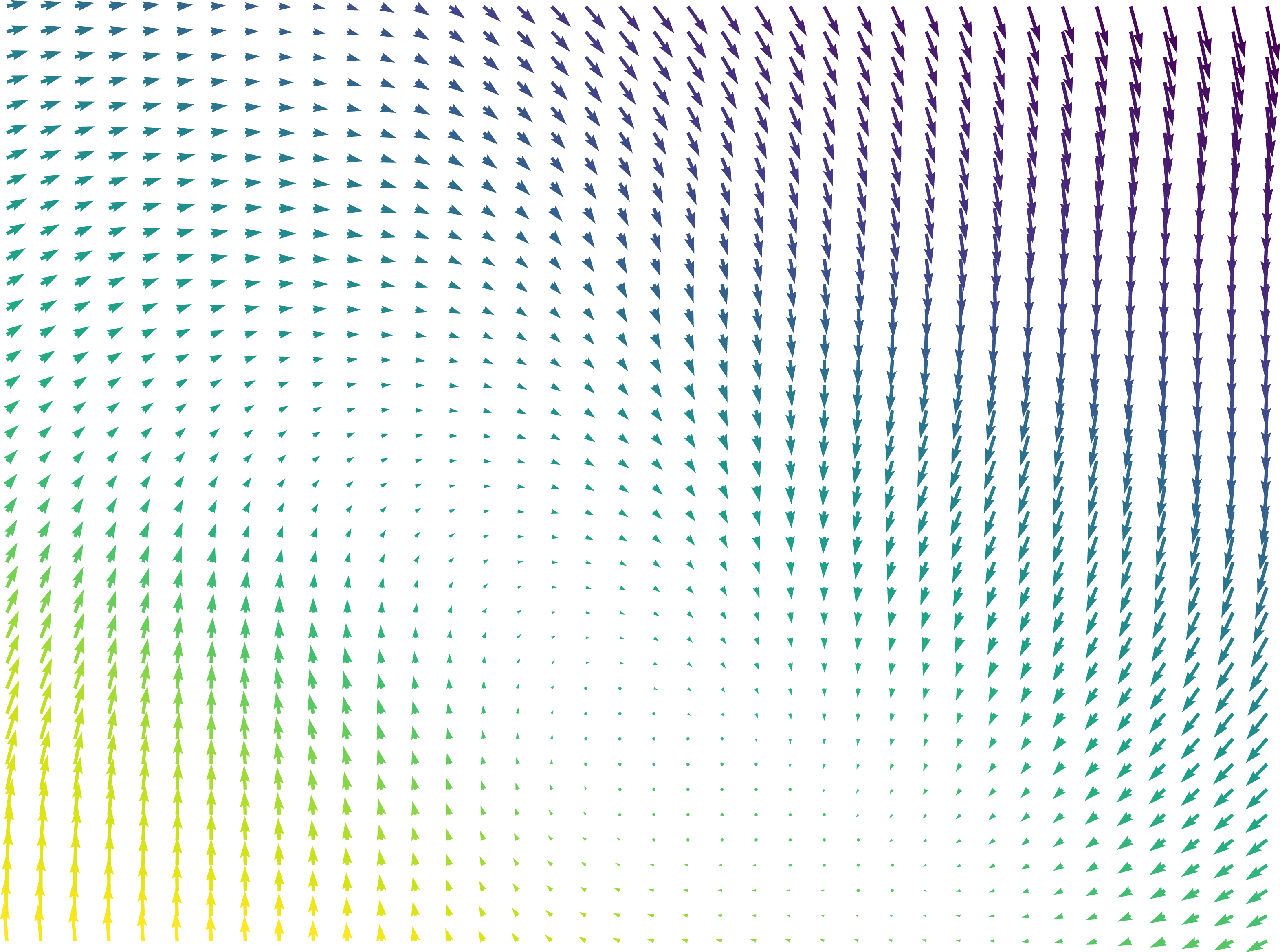}\;}
    \subfloat[SVF]{
    \includegraphics[width=0.31\linewidth]{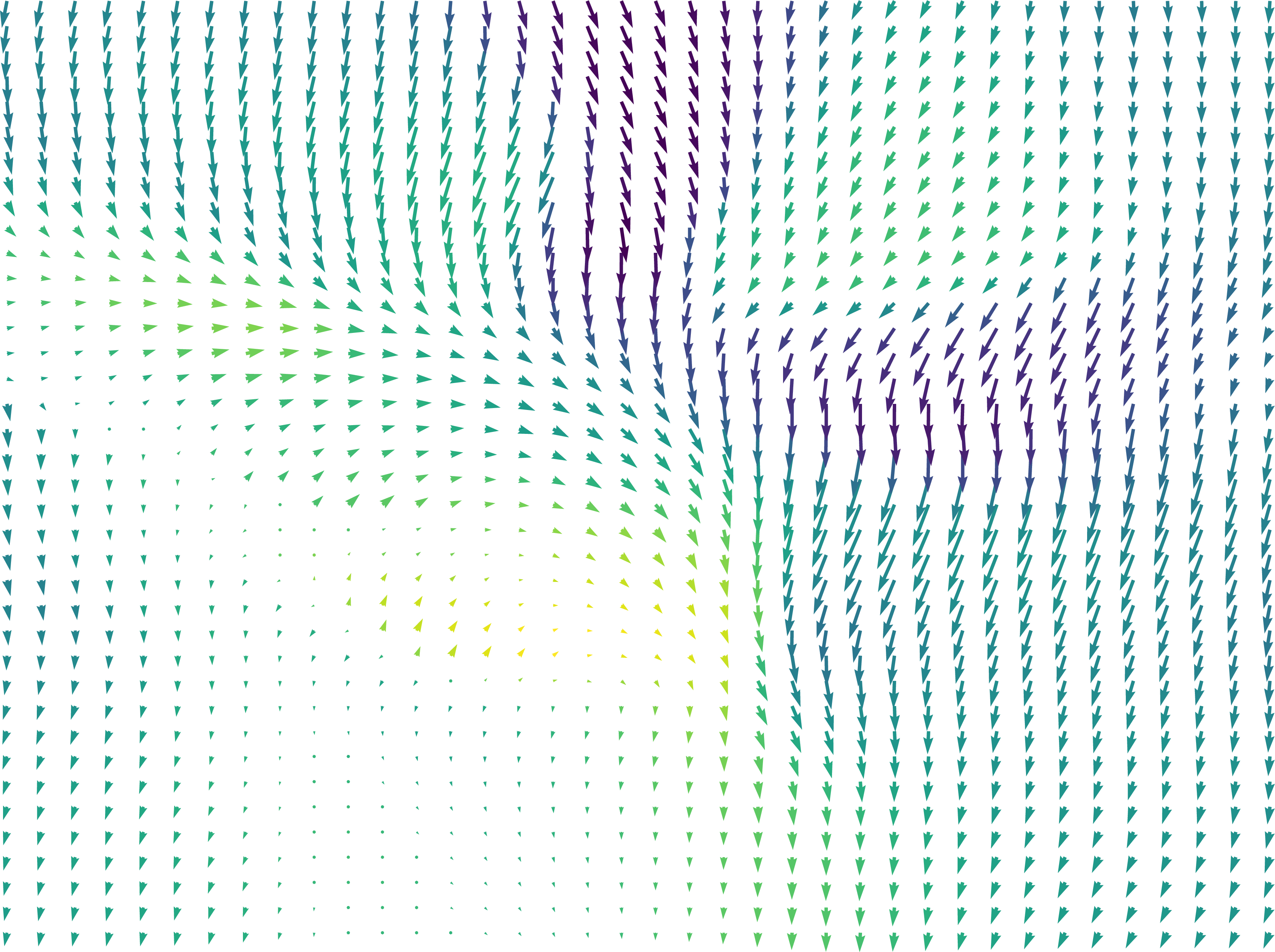}\;}
    \subfloat[SE(3) (ours)]{
    \includegraphics[width=0.31\linewidth]{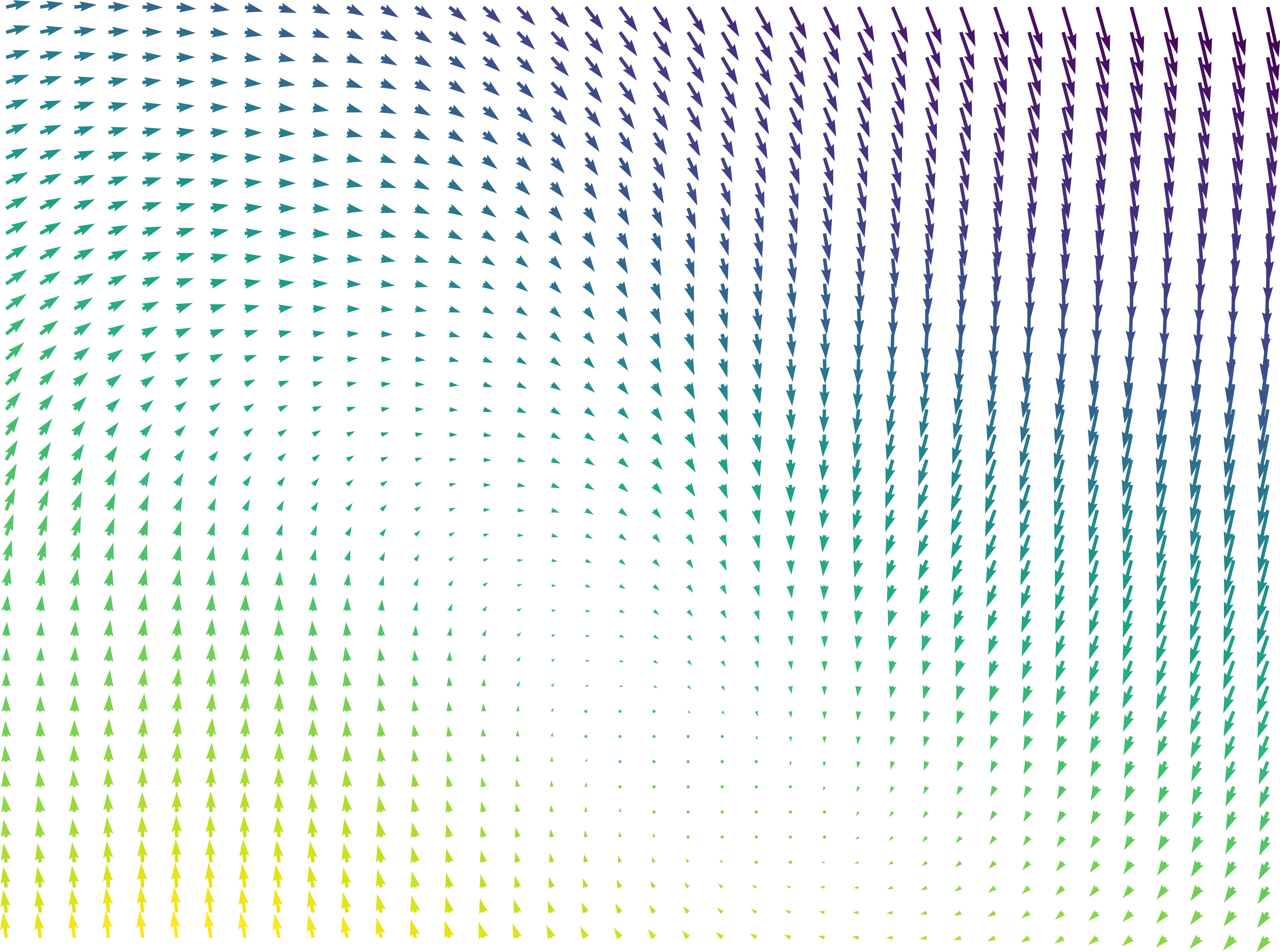}}
    \caption{2D slices of the 3D deformation fields from the bottom row of Fig.~\ref{fig:alignments-iso}. The arrow colors indicate the displacement in the direction orthogonal to the slice. While SVF roughly aligns the images judging in terms of visual quality in Fig.~\ref{fig:alignments-iso}, inspecting the deformation field (b) shows that it is far from the ground truth (a). The proposed parametrization with matrix groups (c) yields a result that closely matches the ground truth.}
    \label{fig:displacement fields}
\end{figure}

To motivate our work, consider the example in Fig.~\ref{fig:alignments-iso}, in which we performed a 3D registration of two MRI brain scans from the OASIS \cite{marcus2007open} dataset using the popular \emph{Stationary Velocity Field (SVF)} approach \cite{arsigny2006log,arguillere2015shape, han2023diffeomorphic} and our proposed method (denoted $\SE(3)$). In the first example, the sought deformation is comparably small before the registration process, and both approaches perform equally well. In the second example, the deformation includes a larger rotational component. When looking at the deformed images only, the SVF approach appears to generate an alignment that is clearly worse, but not catastrophically so. This is deceptive: inspecting the generated deformation fields (Fig.~\ref{fig:displacement fields}), it becomes clear that the SVF approach generates a deformation that -- while it maps corresponding intensity values reasonably well between the images -- is far from the ground truth.

The findings on these (synthetic) examples are compatible with our perception of the existing literature: In general, we found that SVF-based approaches mostly seem to be applied to pre-aligned images, and/or that the judgement of their accuracy is solely based on the similarity of the images after registration or on related proxies, such as the overlap of known segmented regions, that do not necessarily imply sensible deformation fields.

Nevertheless, the SVF approach was used successfully in the past and has many attractive properties, such as the possibility of generating diffeomorphic deformations and of easily obtaining the inverse deformation, for example for bidirectional approaches. Therefore, in this work, we aim to derive an extension that preserves these benefits while allowing for a better handling of large deformations.

\mypar{General approach}
In general, in order to solve \eqref{eq:reg-intro} numerically, a parametrization of the deformation is required, which we denote abstractly by $\phi_\theta$, where $\theta\in\Rn$ is a set of real parameters to be determined for each given image pair $I_1$ and $I_2$.

The registration problem  \eqref{eq:reg-intro} then becomes
\begin{align}
    \min_{\theta \in \Rn} \LL(\phi_\theta;I_1,I_2).\label{eq:reg-param}
\end{align}
Classically, $\phi_\theta$ is either -- in the finite-differences setting -- represented by its values on a grid, taken directly from the parameter vector $\theta$, or -- in a finite elements approach -- is assumed to be a linear combination of basis functions, where $\theta$ determines the coefficients in the linear combination. Typical choices include spline-based parametrizations \cite{szeliski1997spline} and the use of radial basis functions \cite{fornefett2001radial}.

While comparably straightforward to apply, such linear parametrizations have disadvantages: Firstly, it is hard to enforce global regularity of the deformation, in particular invertibility, which is often desirable or even mandated by the physical requirement that the deformation should not contain any ``folds'', i.e., self-intersections. Secondly, due to the local nature of the classical basis functions, components in $\theta$ typically affect only small parts of the deformation, which makes the optimization process more complicated: In order to correctly identify global deformations, a large share of the parameters needs to be adjusted. Furthermore, random local similarities often create local minima or stationary points in the energy, which can trap iterative solvers.

The approach taken in this work builds on the approach used in \cite{han2023diffeomorphic}:
\begin{enumerate}
\item Rather than explicitly parametrizing the deformation, we assume it to be the solution of a \emph{flow equation.} The behavior of the flow equation is governed by its right-hand side in the form of a \emph{velocity field} function $\nu$. Under certain conditions, this setup guarantees that the deformation is sufficiently regular and that an inverse exists which can be computed efficiently.
\item For the parametrization of the velocity field $\nu$, we build on the recent success of coordinate-based networks and assume that the velocity field is described by a network/neural architecture, which is parametrized -- in a nonlinear way -- by $\theta$.
\end{enumerate}

From now on, by $S$ we denote the (non-linear) solution operator that maps such a velocity field to a deformation $\phi = S(\nu)$. The neural network can, in the most generic way, be thought of as a mapping $G$ that takes a parameter vector $\theta$ and generates a velocity field $\nu_\theta = G(\theta)$.

Overall, this results in the model
\begin{align}
    \min_{\theta \in \Rn} \LL(S(G(\theta));I_1,I_2) ,\label{eq:reg-param-full}
\end{align}
where the generator $G$ turns the parameter vector $\theta$ into a velocity field, which is then mapped to a deformation by the solution operator $S$ that solves a \emph{flow equation}. Fig.~\ref{fig:theta-to-phi} visualizes the complete process of generating the deformation $\phi_\theta$ from the parameter vector $\theta$.

\usetikzlibrary{decorations.pathreplacing}
\usetikzlibrary{fadings}
\usetikzlibrary{fit}
\definecolor{fc}{HTML}{1E90FF}
\definecolor{h}{HTML}{228B22}
\definecolor{bias}{HTML}{87CEFA}
\definecolor{noise}{HTML}{8B008B}
\definecolor{conv}{HTML}{b6e1fc}
\definecolor{cyan}{HTML}{87CEFA}
\definecolor{pool}{HTML}{098bdc}
\definecolor{up}{HTML}{B22222}
\definecolor{view}{HTML}{FFFFFF}
\definecolor{bn}{HTML}{FFD700}
\tikzset{fc/.style={black,draw=black,fill=fc,rectangle,minimum height=1cm}}
\tikzset{h/.style={black,draw=black,fill=h,rectangle,minimum height=1cm}}
\tikzset{bias/.style={black,draw=black,fill=bias,rectangle,minimum height=1cm}}
\tikzset{noise/.style={black,draw=black,fill=noise,rectangle,minimum height=1cm}}
\tikzset{conv/.style={black,draw=black,fill=conv,rectangle,minimum height=1cm}}
\tikzset{block/.style={black,draw=black,fill=conv,fill opacity=0.95,rectangle,minimum height=1cm}}
\tikzset{pool/.style={black,draw=black,fill=pool,rectangle,minimum height=1cm}}
\tikzset{up/.style={black,draw=black,fill=up,rectangle,minimum height=1cm}}
\tikzset{view/.style={black,draw=black,fill=view,rectangle,minimum height=1cm}}
\tikzset{conv/.style={black,draw=black,fill=conv,rectangle,minimum height=1cm}}
\tikzset{bn/.style={black,draw=black,fill=bn,rectangle,minimum height=1cm}}
\newcommand*{\eg}{\emph{e.g.}\@\xspace}
\newcommand*{\Eg}{\emph{E.g.}\@\xspace}
\newcommand*{\ie}{\emph{i.e.}\@\xspace}
\newcommand*{\etc}{\emph{etc.}\@\xspace}
\newcommand*{\etal}{\emph{et al.}\@\xspace}
\newcommand*{\cf}{\emph{cf.}\@\xspace}
\newcommand*{\vs}{\emph{vs.}\@\xspace}
\newcommand{\KL}{\ensuremath{\mathrm{KL}}}
\newcommand{\Ber}{\ensuremath{\mathrm{Ber}}}
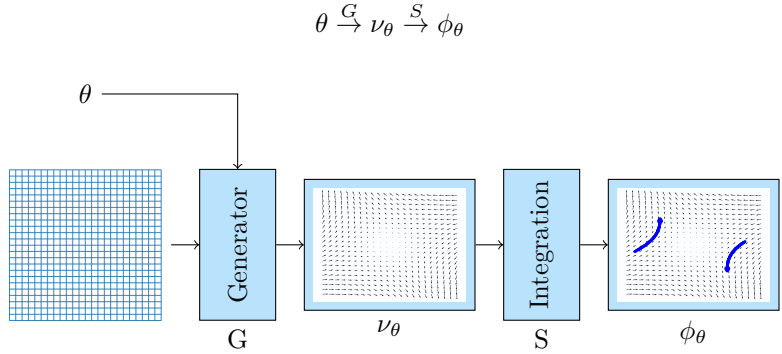
\begin{figure}
    \centering
    \begin{align}
        \theta \overset{G}{\rightarrow} \nu_\theta \overset{S}{\rightarrow} \phi_\theta \nonumber
    \end{align}
    \begin{tikzpicture}
    \node (p) at (0,0) {\includegraphics[width=2cm]{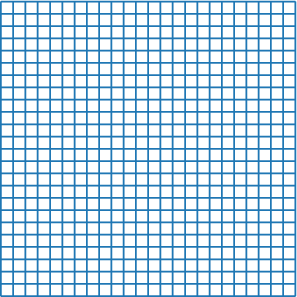}};
    \node[block, minimum width=2cm,rotate=90, label=left:G](G) at (2,0)  {Generator};
    \node[block,rotate=0, label=below:$\nu_\theta$](Vel) at (4,0)  {\includegraphics[height=1.5cm]{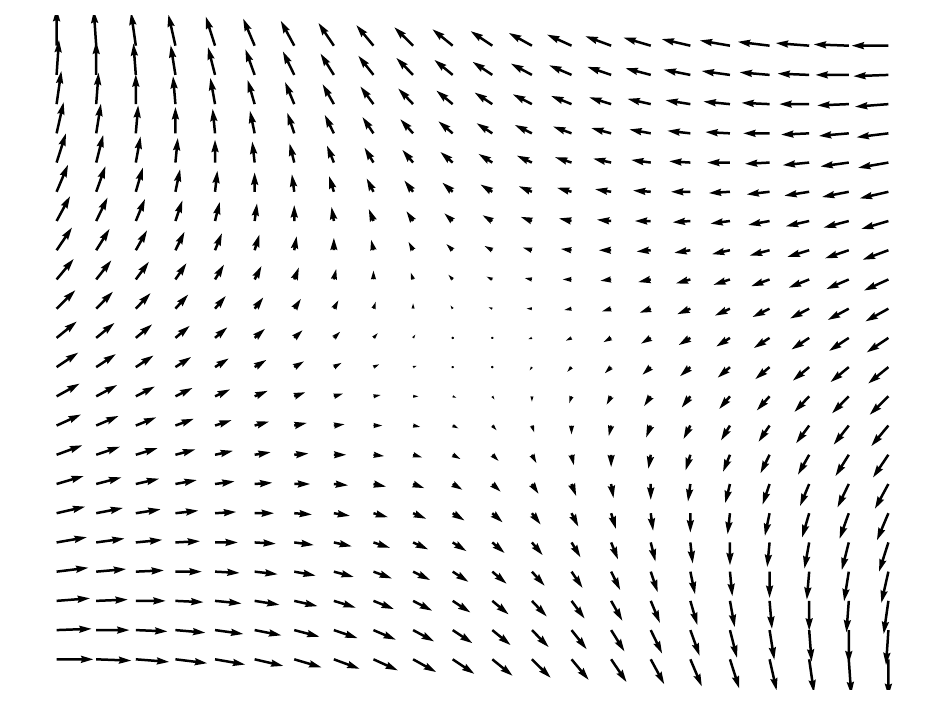}};
    \node[block, minimum width=2cm,rotate=90, label=left:S](S) at (6,0)  {Integration};
    \node[block, label=below:$\phi_\theta$](Phi) at (8,0)  {\includegraphics[height=1.5cm]{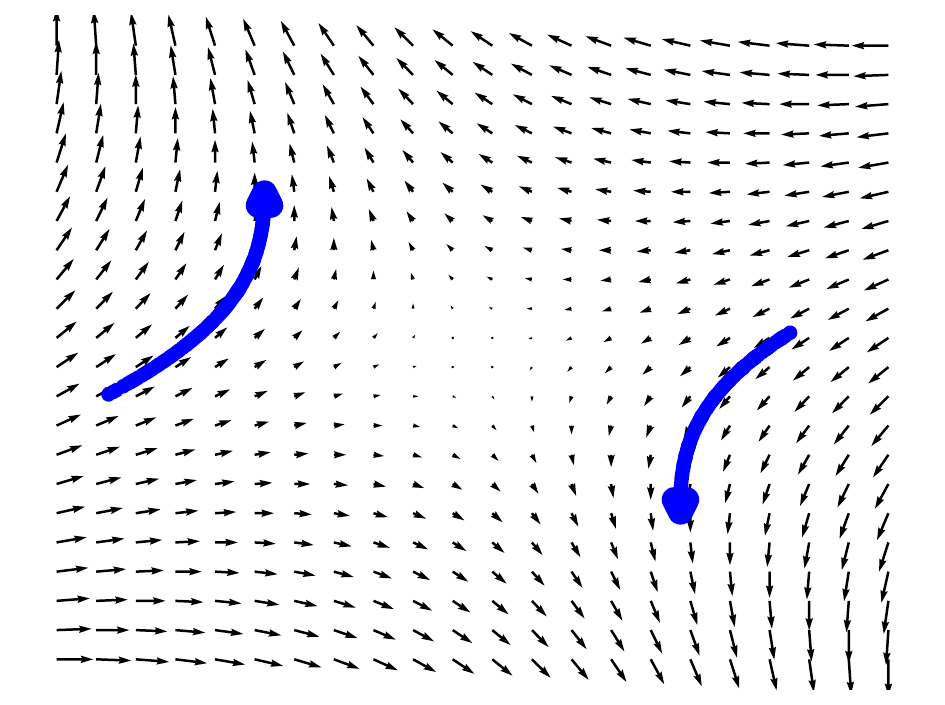}};
    \node[above of = p, node distance =2cm] (theta) {$\theta$};
    \draw[->] (p) -- (G);
    \draw[->] (theta) -| (G);
    \draw[->] (G) -- (Vel);
    \draw[->] (Vel) -- (S);
    \draw[->] (S) -- (Phi);
    \end{tikzpicture}
    \caption{Process of generating a deformation field $\phi_\theta$ from the parameter vector $\theta$. The network-based \emph{generator} $G$ transforms the parameter vector into a \emph{velocity field}~$\nu_\theta$. The \emph{solution operator} $S$ solves an associated flow equation, resulting in a final deformation $\phi_\theta$. Classically, the velocity field and the flow equation are formulated in Euclidean space; we extend the approach to the matrix group setting.}
    \label{fig:theta-to-phi}
\end{figure}

\mypar{Flow equation} In this work, we particularly focus on the flow equation that is used. In the classical setting also used in \cite{han2023diffeomorphic}, one introduces an artificial time $t\in[0,1]$ and prescribes a -- possibly time-varying -- velocity field $v$, for which at each time point~$t$, the associated velocity field $v_t$ is an element of a suitable function space $V$. The deformation is then induced through the ordinary differential equation 
\begin{align}
    \frac{\partial \phi(x,t)}{\partial t} &=  v_t(\phi(x,t)) \text{ for } t \in [0,1], \label{eq1:flow}\\
    \phi(x,0) &=  x,\nonumber
\end{align}
yielding a final deformation $\phi(\cdot,1)$ at time $t=1$.

In the LDDMM approach, this allows to construct a metric on the infinite-dimensional manifold of diffeomorphisms based on properties of $v$, which has also found use as a regularizer \cite{beg2005computing}. Finding that minimizing curves are geodesics with respect to a particular Riemannian metric, a shooting formulation with an initial momentum was subsequently developed \cite{6556700}. Promising a lower computational complexity without significantly hurting expressivity in practical applications, stationary velocity fields also found use \cite{arguillere2015shape}.
These flow-based approaches have been successful in medical image registration, as they enforce diffeomorphic deformations if the velocity field fulfills an integrability condition, 
which can be ensured by a given differentiability operator~\cite{dupuis1998variational,beg2005computing}. 

\mypar{Matrix-valued velocity fields}
A disadvantage of \eqref{eq1:flow} is that the velocity fields that represent common affine deformations such as rotations are non-trivial and require the network to generate such vector fields from scratch.

Therefore, we propose to extend the stationary flow equation \eqref{eq1:flow}
by lifting it to \emph{matrix fields:} Firstly, the deformation is parametrized as $\phi(x):=M(x) \bar{x}$ where $\bar{x}$ is the extension $(x,1)^\top$ of $x$ to homogeneous coordinates, and the matrix field $M:\Omega \to G$ maps to a subgroup $G \subset \GL(\R,d+1)$ of matrices in $\R^{(d+1)\times(d+1)}$. This allows to parametrize common linear transformations, such as rotations, using \emph{constant} matrix fields, which can potentially be learned more easily.

Secondly, in order to formulate the flow equation for these matrix fields, we replace the classical velocity field $v$ by a stationary velocity field $\nu: \Omega \to \mathfrak{g}$ with values in $\mathfrak{g}$, the finite-dimensional vector space of right-invariant vector fields on $G$, that is, the Lie algebra of $G$.

The generalized flow equation then takes the form
\begin{align}
    \frac{\partial M}{\partial t}(x,t) &= \nu(M(x,t) \bar{x})_{M(x,t)} \label{eq2:matflow} \\
    M(\cdot, 0) &= id. 
\end{align}
For the specific choice $G = T(\R,3)$, the group of translations in $\R^3$, the formulation~\eqref{eq2:matflow} reduces to the classical flow equation \eqref{eq1:flow} in the stationary case. 

In a similar spirit, the authors of \cite{park2021nerfies} proposed a parametrization of the deformation using a manifold/$\SE(3)$-valued field, which improved the learning of rotational deformations. This technique shifts the output field into a lower-frequency domain.  However, their description is not flow-based and does not inherently guarantee diffeomorphisms. 

\mypar{Contribution}
To the best of our knowledge, the differentiable structure of matrix groups has not yet been exploited for diffeomorphic, deformable (non-rigid) image registration tasks using neural architectures.
\begin{figure}
\centering   
\begin{tabular}{c|c|c|c}
 Approach & Neural fields & Flow equation & Matrix groups \\
 \hline 
Balakrishnan et al. 19  &   &  \checkmark &  \\
\hline
Han et al. 23  &  \checkmark & \checkmark &  \\
\hline 
Park et al. 21 &  \checkmark &  & \checkmark \\
\hline 
$\emph{ours}$ &  \checkmark & \checkmark & \checkmark 
\end{tabular}
\caption{Components used to parametrize deformations in recent approaches. In this work, we introduce a neural fields based combination of matrix groups with a flow equation.}
\end{figure}
This article is structured as follows.
\begin{itemize}
    \item We propose to extend the existing framework that relies on implicit neural fields and the flow equation by lifting the flow equation to the setting of matrix fields in Section \ref{sec:main}. We prove the existence of a solution of the generalized flow equation in Thm.~\ref{thm:existence} (proof in Section~\ref{sec:existence-proof}).
    \item To obtain a network of manageable depth, in Section~\ref{sec:scaling-squaring} we extend and employ the numerical scaling-and-squaring approach for integrating the flow equation to the matrix group-valued setting. Specifically, we prove that the lifted flow equation satisfies a decomposition proposition that is crucial for the scaling-and-squaring approach (Thm.~\ref{thm:decomp}, proof in Section~\ref{sec:decomp-proof}).
    \item Finally, in Section~\ref{sec4} we confirm numerically that our approach generates invertible deformations, and we validate the accuracy of the generated dense displacement field on synthetic and real-world 3-D registration problems.

\end{itemize}

\mypar{Related Work}  

Supervised learning frameworks for deformable image registration include \cite{haskins2020deep,fu2020deep,sokooti20193d} which are based on artificial deformations and label-driven approaches (also called semi-supervised) \cite{hu2018label}. Ensuring that the artificial training data generalizes well to a specific task is a difficult problem known in the literature as reality gap \cite{tremblay2018training}.

Unsupervised learning-based methods for image registration are typically tuned to minimize a loss based on an image similarity term with possible additional regularisation. 
A successful approach is to start from established image registration approaches that are based on finding a certain optimal set of parameters and replace the optimization process with a trained network: In \cite{balakrishnan2019voxelmorph}, the network is trained to output a stationary velocity field as the basis for generating the deformation in a setting solely based on image similarity and regularisation, as well as in a combined approach using additional segmentation data. In \cite{yang2017quicksilver}, a network is trained to generate initial moments for the shooting formulation of LDDMM based on precalculated moments using a classical approach. In \cite{de2019deep}, the networks directly predict parameters for \mbox{B-spline}-based registration. In \cite{mok2020large}, Laplacian Pyramid Networks are used which directly produce a deformation field. 

A conceptional somewhat different approach uses trained neural architectures as an image similarity metric called ``deep similarity metrics,'' to which classical optimization methods are applied \cite{haskins2020deep}. Combining neural architectures with LDDMM approaches, the authors of \cite{amor2022resnet} proposed a ResNet architecture to model the time-variant flow equation, while \cite{ramon2022lddmm} introduces an adversarial LDDMM learning method. 

Architecturally more akin to classical variational approaches is the concept of neural coordinate representations \cite{sitzmann2020implicit}, in which the unknown function is parametrized non-linearly using a network. There, the estimation of the network parameters is input-specific and generally requires an optimization procedure as in classical variational approaches.
Such approaches work particularly well for novel view synthesis, for which the approach is termed Neural Radiance Fields \cite{mildenhall2021nerf}; see also \cite{park2021nerfies} for an adaptation to deformable scenes. In the context of solving partial differential equations, these networks are known as physics-informed neural networks (PINNs) and have been applied to tasks such as reconstructing flows from 2D observations \cite{cai2021physics} and PDE-constrained optimization \cite{lu2021physics}. Other examples include data compression of images and videos~\cite{sitzmann2020implicit,strumpler2022implicit}.

In the context of deformable image registration, neural representations were studied in~\cite{han2023diffeomorphic,wu2022nodeo} in a diffeomorphic setting with stationary velocity fields and, among others, in~\cite{wolterink2022implicit} for implicit regularisation. 

Our proposed approach is based on a specific representation of transformations in the form of matrix groups. Transformations of the Euclidean space that preserve specific structures have already been studied and categorized for a long time. Since the 19th century, continuous transformation groups were considered with a differentiable structure, also known as infinitesimal transformations. This effort culminated in a full-fledged theory of Lie groups and their corresponding Lie algebras \cite{hawkins2012emergence}. Since then, the theory found use in applications including quantum theory \cite{iachello2006lie}, computer graphics~\cite{eade2014lie, xu2012applications}, odometry \cite{loianno2016visual}, and robotics \cite{park1995lie,selig2013geometrical}. 
In the context of machine learning, matrix groups were often employed to ensure specific invariances -- for example, in classification tasks~\cite{moskalev2022liegg} -- or equivariances~\cite{finzi2020generalizing}.
The authors of~\cite{teed2021tangent} introduced tangent space backpropagation for automatic differentiation on Lie groups including $\SO(3)$, $\SE(3)$, and $\SIM(3)$. Combining image registration with transformation groups, the authors of~\cite{park2021nerfies} proposed parametrizing a deformation pointwise by members of $\SE(3)$ to better represent rotational deformations.

\section{Mathematical preliminaries}\label{sec2}
\begin{table}[tb]
        \centering
        \begin{tabular}[width=\linewidth]{ p{1.7cm} p{1.7cm} p{1.7cm} p{1.7cm} p{1.7cm}p{1.7cm} }
        \toprule
        Matrix group & $\Tgp(\R,3)$ & $\SE(\R,3)$ &  $\SIM(\R,3)$  & $\Aff(\R,3)$ & $\PGL(\R,3)$\\  
        Dimension: & 3 & 6 & 7 & 12 & 15 \\
        Transform: & Translation & Rigid & Rigid, scaling & Affine & Perspective \\
        Preserves: & Orientation~\& & Area \&  & Angles \& & Parallels,\linebreak ratio of distances \& &  Collinearity, cross-ratio\\
        Known closed form of $\exp;\log$: & yes & yes & yes & no & no\\
        \bottomrule
        \end{tabular}
    \caption{Common matrix groups acting on $\R^3$.}
    \label{tab:lie_groups}
\end{table}
A real \emph{Lie group} is a group $(G,\cdot)$ that is also a smooth real \emph{manifold} such that the following conditions hold \cite{lee2012smooth}:
\begin{itemize}
    \item The mapping $\cdot: G \times G \to G : (g,h) \mapsto g\cdot h$ is continuous.
    \item The mapping $()^{-1}:G \to G: g \mapsto g^{-1}$ is continuous.
\end{itemize}
Table~\ref{tab:lie_groups} lists the Lie groups that are most relevant to this work, such as the special Euclidean group of rigid transformation of points in $\R^3$, $\SE(3)$, which can be thought of as a subset of real $4\times 4$ matrices when using the usual homogeneous coordinate representation for the points in $\R^3$. Such groups that are representable by (invertible) matrices are also known as $\emph{matrix groups}$. 

All groups in Table~\ref{tab:lie_groups} can thus be embedded in the 16-dimensional vector space $\R^{4\times4}$. In fact, Whitney's embedding theorem \cite{Mukherjee2015ApproximationTA} guarantees that, under weak theoretical conditions, each $d$-dimensional smooth real manifold can be embedded into a $2d$-dimensional Euclidean space. 

Therefore, the \emph{tangent space} at a point $p \in M$ can be visualized as an affine subspace in $\R^{4\times4}$ that touches and linearly approximates the manifold at $p$; a tangent vector is then a vector attached to $p$ that lies in this affine subspace. 

More rigorously, the tangent space $T_pG$ at a point $p \in G$ is defined here as a set of equivalence classes of smooth curves on the manifold $\gamma: I \subset \R \to G, \, \gamma(0)=p$ with the equivalence relation 
\begin{align}
    \gamma_1 \sim \gamma_2  \Leftrightarrow  (\phi \circ \gamma_1)'(0)=(\phi \circ \gamma_2)'(0) 
\end{align}
for every chart $\phi: U_p \to \R^k$ defined  on a neighborhood $U_p$ of $p$. If the manifold can be embedded into a vector space $\R^l$, the charts can be replaced by embeddings $i: U_p \to \R^l$ \cite{boumal2023intromanifolds}.

The disjoint union of all tangent spaces $TG= \bigsqcup_{p \in G} T_p G$ is called a \emph{tangent bundle.} A \emph{vector field} $X: G \to TG$ with $Xp \in T_pG$ is a section through the tangent bundle in the sense that it assigns to each point in $G$ a vector in the its tangent space. 

Given a differentiable map on the manifold $\chi: G \to G$, its \emph{push-forward} or \emph{differential} $\mathcal{D\chi}_p: T_p \to T_{\chi(p)}$ maps between corresponding tangent spaces. This operation can be defined at a point $p \in G$ by mapping a representative  $\gamma$ from $T_p$ to equivalent curves of $\chi \circ \gamma$, i.e.,
\begin{align}
    D(\chi)_p(\gamma)= \chi\circ\gamma .
\end{align}

A particular case of smooth vector fields is the class of \emph{left/right invariant} (smooth) vector fields, which is of interest for the following chapters. It is reasonable to consider the smooth and bijective maps induced by right/left 
multiplication of group elements~$l_g: G \to G, \, \mapsto g\cdot h$ and $r_g: G\to G, \, h \mapsto h\cdot g$, respectively. Such invariant vector fields fulfill the criteria
\begin{align*}
    \text{left invariance: }& \quad\mathcal{D}_{l_g} X_p=X_{g\cdot p} \\ \text{ right invariance: }& \quad\mathcal{D}_{r_g}X_p=X_{p\cdot g}
\end{align*}
for all $p,g \in G$, where $D_{l_g}$ describes the push-forward/differential of the left or correspondingly right translation.

In general, the space of vector fields is a function space, which makes finding a numerical representation difficult. The class of invariant vector fields, however, can be uniquely described by a single value: their value at the identity. For the finite-dimensional manifolds in this work, the space of such invariant vector fields is finite-dimensional. This greatly simplifies numerical treatment and allows for generating invariant vector fields using a coordinate-based network in which each output channel corresponds to one dimension of the matrix group.

An example of a right invariant vector field for the commutative one-dimensional group $\SO(2)$ of two-dimensional rotation matrices, which is topologically equivalent to a circle, is given in Figure~\ref{fig:left-invariant vector fields on SO2}. The right invariance of the vector field translates to the condition of the tangential vectors having the same signed length. 
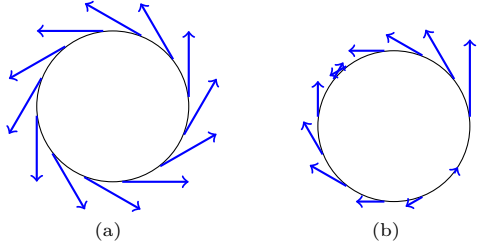
\begin{figure}
    \centering
    \subfloat[]{
    \begin{tikzpicture}
 [
   scale=0.5,
   point/.style = {draw, circle, fill=black, inner sep=0.5pt},
 ]

\def\rad{2cm}
\node (C) at (0,0) {};
\draw (C) circle (\rad);
\node (P)  at +(0:\rad)  {};
\node (Q)  at +(30:\rad) {};
\node (R)  at +(60:\rad) {};
\node (S)  at +(90:\rad) {};
\node (T)  at +(120:\rad) {};
\node (U)  at +(150:\rad) {};
\node (V)  at +(180:\rad) {};
\node (W)  at +(210:\rad) {};
\node (X)  at +(240:\rad) {};
\node (Y)  at +(270:\rad) {};
\node (Z)  at +(300:\rad) {};
\node (A)  at +(330:\rad) {};

\draw[->,thick, color=blue] (P) -- ([turn]90:2cm);
\draw[->,thick, color=blue] (Q) -- ([turn]90:2cm);
\draw[->,thick, color=blue] (R) -- ([turn]90:2cm);
\draw[->,thick, color=blue] (S) -- ([turn]90:2cm);
\draw[->,thick, color=blue] (T) -- ([turn]90:2cm);
\draw[->,thick, color=blue] (U) -- ([turn]90:2cm);
\draw[->,thick, color=blue] (V) -- ([turn]90:2cm);
\draw[->,thick, color=blue] (W) -- ([turn]90:2cm);
\draw[->,thick, color=blue] (X) -- ([turn]90:2cm);
\draw[->,thick, color=blue] (Y) -- ([turn]90:2cm);
\draw[->,thick, color=blue] (Z) -- ([turn]90:2cm);
\draw[->,thick, color=blue] (A) -- ([turn]90:2cm);
\end{tikzpicture}} 
\hspace{25pt}
\subfloat[]{
    \begin{tikzpicture}
 [
   scale=0.5,
   point/.style = {draw, circle, fill=black, inner sep=0.5pt},
 ]

\def\rad{2cm}
\node (C) at (0,0) {};
\draw (C) circle (\rad);
\node (P)  at +(0:\rad)  {};
\node (Q)  at +(30:\rad) {};
\node (R)  at +(60:\rad) {};
\node (S)  at +(90:\rad) {};
\node (T)  at +(120:\rad) {};
\node (U)  at +(150:\rad) {};
\node (V)  at +(180:\rad) {};
\node (W)  at +(210:\rad) {};
\node (X)  at +(240:\rad) {};
\node (Y)  at +(270:\rad) {};
\node (Z)  at +(300:\rad) {};
\node (A)  at +(330:\rad) {};

\draw[->,thick, color=blue] (P) -- ([turn]90:2.3cm);
\draw[->,thick, color=blue] (Q) -- ([turn]90:1.8cm);
\draw[->,thick, color=blue] (R) -- ([turn]90:1.4cm);
\draw[->,thick, color=blue] (S) -- ([turn]90:1.2cm);
\draw[->,thick, color=blue] (T) -- ([turn]90:0.8cm);
\draw[->,thick, color=blue] (U) -- ([turn]90:-0.8cm);
\draw[->,thick, color=blue] (V) -- ([turn]90:-1.2cm);
\draw[->,thick, color=blue] (W) -- ([turn]90:-1.3cm);
\draw[->,thick, color=blue] (X) -- ([turn]90:-1.4cm);
\draw[->,thick, color=blue] (Y) -- ([turn]90:-1cm);
\draw[->,thick, color=blue] (Z) -- ([turn]90:-0.8cm);
\draw[->,thick, color=blue] (A) -- ([turn]90:-0.1cm);
\end{tikzpicture}} 
\caption{(a) Visualization of a right-invariant vector field on $SO(\R,2)$, homeomorphically identified with a circle. 
(b) Visualization of a non-invariant vector field on $SO(\R,2)$. Right-invariant fields on $SO(\R,2)$ can be parametrized with a single value; this is, in general, not possible for non-invariant vector fields.} 
\label{fig:left-invariant vector fields on SO2}
\end{figure}

Curves $\gamma:\R \to G$ solving the differential equation
\begin{align}
    \frac{\partial \gamma}{\partial t}(t) &= X_{\gamma(t)}\\
    \gamma(0) &= g_0
\end{align}
are called \emph{integral curves} originating from $g_0$ \cite{lee2012smooth}.

In the case of embedded matrix groups, integral curves of right invariant vector fields $X$ take the form
\begin{equation}
    \gamma(t)= \exp(t \cdot X_{\id})g_0 \label{eq: integral curve}
\end{equation}
where ``$\exp$'' denotes the matrix exponential
\begin{equation}
    \exp(M) := \sum_{i\in \N} \frac{M^i}{i!}.
\end{equation}
This can be shown by differentiating the curve \eqref{eq: integral curve} with respect to time:
\begin{equation}
    \frac{\partial \gamma}{\partial t}(t) = 
    \exp(t X_{\id}) X_{\id} g_0 = X_{\id} \exp(t X_{\id})  g_0 = X_{\exp(t X_{\id})g_0} = X_{\gamma(t)}.  
\end{equation}

In Section \ref{sec:scaling-squaring}, the matrix exponential will be used to solve the extended flow equation on matrix groups numerically.

\section{Image registration using velocity flow on matrix group valued fields}\label{sec:main}

Before going into the analysis, we briefly summarize our approach as outlined in the introduction. We consider an overall model of the form
\begin{align}
    \min_{\theta \in \Rn} \LL(S(G(\theta));I_1,I_2).\label{eq:reg-param-main}
\end{align}
The final deformation $\phi_\theta:=S(G(\theta))$ is parametrized by the parameter vector~$\theta$. This parameter vector is first turned into a stationary velocity field $\nu_\theta:=G(\theta)$ by the neural network $G$. The solution operator $S$ then computes the final state $M(\cdot,1)$ of the flow equation
\begin{align}
    \frac{\partial M}{\partial t}(x,t) &= \nu_\theta(M(x,t) \bar{x})_{M(x,t)}\label{eq:matflow2-main} \\
    M(\cdot, 0) &= id\nonumber
\end{align}
and returns the final deformation
\beq
    \phi_\theta(x):=S(\nu_\theta)(x):=M(x,1)\begin{pmatrix}
        x\\
        1\\
    \end{pmatrix}.\label{eq:hom-transform}
\eeq

In the following, we discuss the individual parts of this approach in more detail:
\begin{itemize}
\item Sect.~\ref{sec:flow-existence} covers the generalized flow equation and existence,
\item Sect.~\ref{sec:scaling-squaring} addresses the question on how to efficiently implement the solution operator $S$ of the flow equation,
\item Sect.~\ref{sec:objective} concerns the choice of similarity measure $J$ and regularizer $R$ that form the objective $\LL$, and
\item in Sect.~\ref{sec:network-architecture}, we discuss the network architecture $G$.
\end{itemize}

In Sect. ~\ref{sec4}, we discuss numerical results for synthetic deformations and for inter-patient registration tasks.

The idea behind the generalized flow equation is to obtain more freedom and control in tuning the hyperparameters to create a bias so that specific (large) deformations become easier to learn, while preserving the theoretical features of a flow-based approach, namely ``nearly'' diffeomorphic deformations and a simple method for calculating the inverse deformation.

\subsection{Extended flow equation and existence}\label{sec:flow-existence}
\mypar{Velocity fields and vector fields} The starting point for including the flow equation~\eqref{eq:matflow2-main} in our approach is the classical \emph{stationary velocity field} approach, in which the flow equation takes the form 
\begin{align}
    \frac{\phi(x,t)}{\partial t} &=  v(\phi(x,t)) 
    \label{eq:flow2}\\
    \phi(x,0) &=  x . \nonumber
\end{align}

In contrast to a time-dependent velocity field $v_t$ as in \eqref{eq1:flow}, this formulation reduces the computational effort \cite{yang2015diffeomorphic} notably. It is used in classical frameworks \cite{yang2015diffeomorphic} as well as frameworks using a neural architecture \cite{balakrishnan2019voxelmorph,han2023diffeomorphic}.
If the velocity field $v$ is an element of $C_0^1(\Omega)$, the solution $\phi(\cdot,1)$ is continuously differentiable and has a continuously differentiable inverse, i.e., it belongs to the class of \emph{diffeomorphisms.} This is a direct consequence of the results in~\cite{arguillere2015shape}, where the more general case of time-varying velocity fields was considered.

Diffeomorphic deformations are often desirable for registration tasks in which the deformation is caused by a physical process, as is commonly the case in biomedical data. This is due to the fact that non-diffeomorphic deformations, in particular self-intersections, typically are physically implausible. Unfortunately, not every diffeomorphism can be obtained as a solution of \eqref{eq1:flow} with some stationary velocity field \cite{lorenzi2013geodesics}; however, known exceptions seem rather implausible for anatomical deformations \cite{lorenzi2013geodesics}.

Our approach \eqref{eq:matflow2-main} is based on the idea that velocity fields that correspond to basic, expected deformations -- in particular, translations and rotations -- should be ``simple'' in the sense that they are low-frequency, ideally even constant, and therefore easy to generate. For the standard SVF model \eqref{eq:flow2}, constant velocity fields can only model homogeneous translations of the whole domain. By employing \emph{matrices} $M(x,t)$ instead of vectors $v(x,t)$ as in \eqref{eq:matflow2-main} and constructing the final deformation as in \eqref{eq:hom-transform}, constant velocity fields can also model affine deformations.

In order to restrict the space of possible matrices, we require that all $M(x,t)$ are elements of a \emph{matrix group} $G$, i.e., of a continuous subgroup of the general linear group $\GL(4,\R)$. With the usual matrix multiplication as group operation, these groups are also Lie groups with a differential (manifold) structure. 

Then, the time derivative $\frac{\partial}{\partial t} M(x,t)$ in the flow equation \eqref{eq2:matflow} is an element of the tangent space $T_{M(x,t)} G$.  This requires the velocity field $\nu_\theta$ on the right-hand side to be an element of $T_{M(x,t)} G$, which is problematic, as our goal is to parametrize the velocity field by a neural network: For a general manifold, the parametrization of the tangent space depends on the location $M(x,t)$, which is not known when deciding on the structure of the network, and furthermore changes with time.

The key is to restrict the vector fields 
to the space $\mathfrak{g}$ of \emph{right-invariant vector fields} on $G$. An element $\nu(x)=:\nu' \in \mathfrak{g}, \, \nu': G \to TG$ of the space of right-invariant vector fields assigns to each point $p \in G$ a vector in the tangent space $T_pG$, denoted by $\nu'_g$. Most importantly, the right invariance assures that each such vector field~$\nu'$ is uniquely described by its value at the identity $id$. Therefore, the neural network can be constructed to map into the tangent space $T_{\id} G$ of $G$ at the identity.

As an example, consider $SO(2)$, the group of rotations around the origin in $\R^2$. Any such rotation is uniquely parametrized by a rotation angle $\alpha$. Solving the flow equation~\eqref{eq:matflow2-main} with a right-invariant vector field would take any such rotation and increase or decrease the rotation angle by the same amount, irrespective of the angle of the original rotation (Fig.~\ref{fig:left-invariant vector fields on SO2}).

As the velocity field $\nu$ is stationary, but not necessarily constant, this formulation may also capture non-rigid deformations, even when the chosen matrix group only corresponds to (a subgroup) of rigid deformations. 
Setting $G$ to the group of translations~$T(\R)$, this formulation reduces to the previous SVF setting of~\eqref{eq1:flow}.

If $G$ contains the subgroup of translations, then the space of possible deformations is at least as large as the classical SVF approach.
However, choosing a more expressible matrix group may be seen as an over-parametrization of the resulting deformation, as different velocity fields may lead to the same induced deformation. For example, in case of $G=\SE(2)$, a rotation around the origin can either be achieved by a constant field $\nu$ in which only the angular velocity is non-zero, or by a spatially varying ``swirl'' of the translational components of $SE(2)$.

\mypar{Existence and smoothness}
We first justify the use of the generalized SVF/flow equation \eqref{eq:matflow2-main} by showing the existence and smoothness of the solutions. Note that $\nu_{\id}$ maps from the image domain $\Omega$ into the space of right-invariant vector fields $\mathfrak{g}$ \emph{evaluated at $id$}, so that $\nu_{\id}(x)\in T_{\id} G$.

We are interested in a 3D registration setting, so we now consider subgroups of the affine group of $\R^3$ for $G$. Its elements can be represented as $4 \times 4$ matrices.

\begin{theorem}\label{thm:existence}

Let $\Omega$ be compact, let $\mathfrak{g}$ denote the space of right-invariant vector fields on $G \subset \GL (\R, 4)$, and let $\nu: \Omega \to \mathfrak{g}$ be such that its evaluation at the identity $\nu_{\id}$ is Lipschitz continuous.
Then the solution $M$ of \eqref{eq2:matflow} exists uniquely and remains bounded in $C(\Omega,\R^{4\times 4})$ when viewed as a map $t\mapsto M(\cdot,t)$.
\end{theorem}
\begin{proof}
See Section~\ref{sec:existence-proof}.
\end{proof}

\subsection{Numerical integration by scaling and squaring}\label{sec:scaling-squaring}
\begin{figure}[tbp]
\begin{minipage}{0.48\textwidth}
\begin{algorithm}[H]
\footnotesize
\caption{Scaling and Squaring} \label{alg:SnQ}
\begin{algorithmic}
\State $v_0 \gets 2^{-n}v$  \hfill \Comment{scaling}
\For{$j \in \{0,...,n-1\}$ }
    \State $v'_j \gets \text{Int}(v_j,x+v_j(x))$
    \State $v_{j+1} \gets  v'_j + v_j$ \hfill \Comment{squaring}
\EndFor
\State \Return $v_{j+1}$  
\vspace{8pt}
\end{algorithmic}
\end{algorithm}
\end{minipage}
\hfill
\begin{minipage}{0.50\textwidth}
\begin{algorithm}[H]
\footnotesize
\caption{Generalized Scaling and Sq.}\label{alg:exSnQ}
\begin{algorithmic}
    \State $\nu_{0} \gets 2^{-n} \nu $ \hfill \Comment{scaling}
    \For{$j \in  \{0,...,n-1\}$ }
        \State $M_j \gets \exp(\nu_{j})$
        \State $\nu'_{j} \gets \text{Int}(\nu_{j},PM_{j}\barx)$
        \State $\nu_{j+1} \gets \log(\exp(\nu_{j}')\cdot M_{j})$ \hfill \Comment{squaring}    
    \EndFor
    \State \Return $\exp(\nu_{j+1})$ 
\end{algorithmic}  
\end{algorithm}
\end{minipage}
\caption{Left: Classical scaling-and-squaring approach for fast numerical integration of the flow equation \eqref{eq1:flow} based on a given velocity field $v$. The interpolation operator $\text{Int}(v,x)$ computes an approximation of $v(x)$. Right: Proposed generalization to the matrix-valued setting \eqref{eq:matflow2-main} for given velocity field $\nu$. Again the interpolation operator $\text{Int}(\nu,x)$ approximates $\nu(x)$.}
\end{figure}
In the classical setting, stationary velocity fields allow the use of the \emph{scaling and squaring} approach~\cite{arsigny2006log} to approximate the solution of the ODE \eqref{eq1:flow} numerically. As can be seen from Alg.~\ref{alg:SnQ}, the first step is essentially a forward Euler step with step size~$2^{-n}$. This is followed by the squaring steps $v_{j+1} \gets v_j' + v_j$ which correspond for $k=n-j$ to the ``decomposition condition''
\begin{equation}
    \phi \left(x, 2^{-k+1} \right) = \phi \left( \phi(x, 2^{-k}), 2^{-k} \right) . \label{eq:decompcondgeneral}
\end{equation}
Intuitively, the time resolution scales \emph{exponentially} in $n$, making this approach computationally attractive.

As we aim at including the solution operator into a network-based optimization scheme, this idea is crucial for limiting the network depth. Therefore, we extend the technique to the setting of matrix groups in order to approximate the solution to \eqref{eq:matflow2-main} effectively in a discretized setting. The complete algorithm is shown in Alg.~\ref{alg:exSnQ}.
The final deformation field at $t=1$ is calculated iteratively, starting with a simple forward Euler step
\begin{align}
   M(x,2^{-n}) &= \exp \left(2^{-n}\nu_{\id} \left(x \right) \right) \label{eq:sc_and_sq}
\end{align}
and iterating over $n$ squaring steps.
The algorithm is based on the ``Exponential discretization'' scheme \eqref{eq:expscheme}, which is further discussed in Section~\ref{sec:decomp-proof}. There, we also prove the convergence of the semi-discrete scheme to the continuous solution of \eqref{eq:matflow2-main}.

For the matrix-valued deformations considered here, the ``decomposition condition'' \eqref{eq:decompcondgeneral} takes the form
\begin{align}
   M(x,2^{-k+1}) &= M \left( PM\left(x, 2^{-k} \right)\barx, 2^{-k}\right) M\left(x, 2^{-k}\right) ,\label{eq:decompcond}
\end{align} 
where $P:\R^4\to\R^3$ denotes the projection which removes the last component.

In order to successfully apply the scaling-and-squaring scheme, it is crucial that the solution of the flow equation satisfies the decomposition condition~\eqref{eq:decompcond}. This is guaranteed by the following theorem.

\begin{theorem}\label{thm:decomp}
    Let $M$ be a solution of the flow equation \eqref{eq2:matflow}. Assume that $\Omega\subset\R^3$ is compact and that the velocity term $\nu : \Omega \to \mathfrak{g}$ satisfies $\nu_{\id} \in C^1(\Omega, \R^{4\times 4})$.
    Then $M$ satisfies the decomposition condition 
    \begin{align}
       M(x,2T)   = M(PM(x, T)\Bar{x}, T)M(x, T) \quad \text{for all } T>0 . \label{eq:decomp}
    \end{align}
    Furthermore, when viewed as a map $t\mapsto M(\cdot,t)$,
    it holds that $M~\in~\mathcal{C}^2([0,1],C(\Omega,\R^{4\times 4}))$.
where we identify $M(x,t)$ with $M(t)(x)$ as required.
\end{theorem}
\begin{proof}
See Section~\ref{sec:decomp-proof}.
\end{proof}

When applying the scaling-and-squaring scheme for integrating the flow equation, there are some implementation details to consider:

\begin{itemize}
\item \mypar{Network depth} Each squaring operation doubles the effective resolution in the synthetic time variable. Setting $n=0$ reduces the method to the direct/non-flow-based approach, where the neural network directly outputs the values for the displacement vector field. Larger values of $n$ yield a better numerical approximation of the ODE, at the cost of network depth and additional computational effort. In this work, $n$ is always set to $7$, which yields an implicit time discretization of $\frac{1}{128}$.
\item \mypar{Interpolation} The scaling-and-squaring approach relies on the efficient evaluation of the solution at the current time and at arbitrary points in space. When discretizing the matrix field on a spatial grid, this requires an interpolation between elements of the given matrix group. As the matrix group is generally curved and non-convex, multi-linear interpolation in the surrounding vector space does not ensure that the matrix-valued fields map to $G$. Instead, after interpolation, the matrix fields would, in general, map into $\R^{4\times 4}$. Therefore, we resort to interpolation on the space of right invariant vector fields on $G$, which carries vector space structure. The use of the matrix logarithm, which can be calculated analytically for $\SE(3)$ and $\SIM(3)$, allows us to carry out the interpolation in the Lie algebra. The result is then mapped back to the matrix group using the matrix exponential. This step guarantees that the interpolated result is still an element of the given matrix group $G$.
\item \mypar{Differentiability} The matrix logarithm and exponential can be implemented in a fully differentiable way; we rely on the implementations for $\SE(3)$ and $\SIM(3)$ in~\cite{teed2021tangent}. 
The necessary interpolation creates an additional numerical error depending on the smoothness of the velocity field. With our approach using matrix-valued fields, specific deformations, such as rigid deformations for $G=\SE(3)$, can be recreated without these numerical approximation errors, which might be advantageous.
\item \mypar{Boundary conditions} Deviating from the usual diffeomorphic approach, we do \emph{not} fix the boundary but constantly extend the velocity field by its boundary values outside of the domain with the value at the border. This allows to reconstruct translations or rigid deformations in the $\SE(3)$ setting exactly.
\end{itemize}

\mypar{Inverse deformation and validation}%
\begin{figure}[tbp]
    \centering
     \subfloat[]{\includegraphics[width=0.48\linewidth]{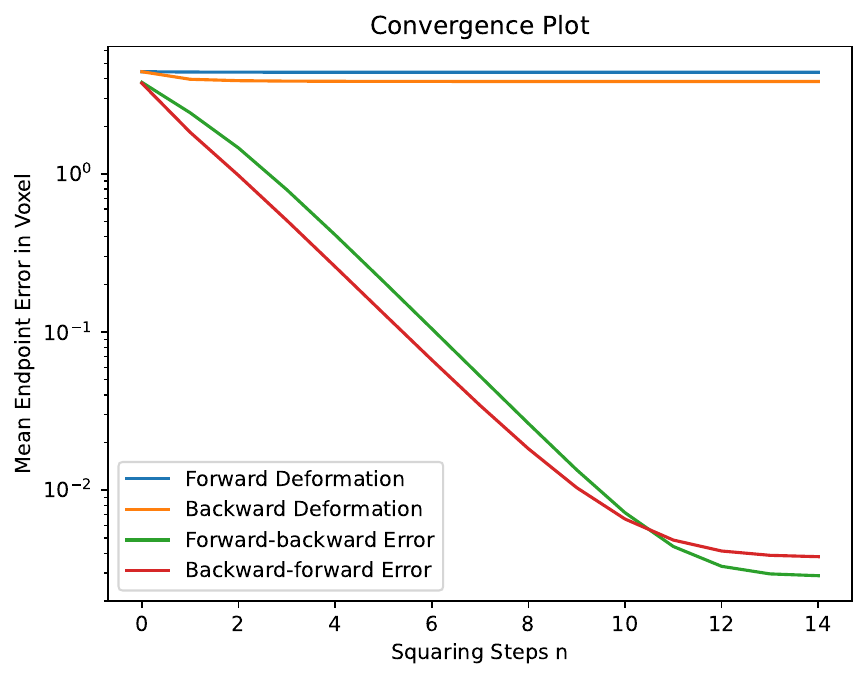}}
     \subfloat[]{\includegraphics[width=0.48\linewidth]{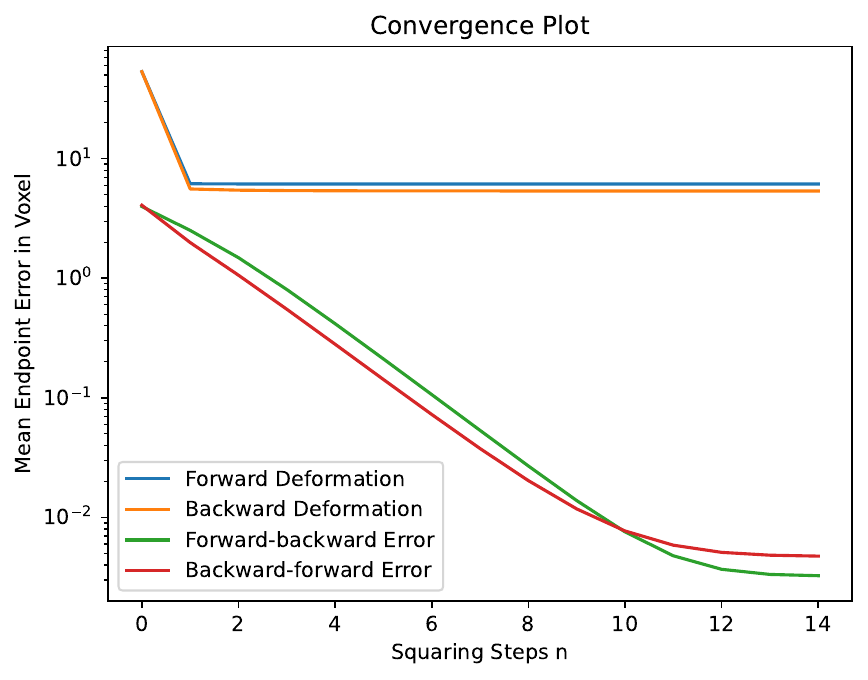}}
    \caption{Residual errors depending on the number of scaling and squaring steps for a given velocity field for (a) classical integration of the flow equation in Euclidean space and (b) integration on $\SE(3)$ using the modified flow equation.}
    \label{fig:scaling-squaring Convergence}
\end{figure}%
Taking the interpolation into account, it is prudent to ask whether the proposed scaling-and-squaring approach generates reasonably accurate deformations in practice. To approach this question, note that if
\begin{align}
    \phi=S(\nu_\theta),
\end{align}
i.e., $\phi$ solves the flow equation for given velocity field $\nu_\theta$, then the inverse $\phi^{-1}$ solves the flow equation for $-\nu_\theta$:
\begin{align}
    \phi^{-1}=S(-\nu_\theta).\label{eq:inverse-property}
\end{align}
This suggests a method for verifying the consistency of the numerical (scaling-and-squaring) integration scheme: For a given velocity field $\nu_\theta$, compute both $\phi_{\nu_\theta}$ as well as $\phi_{-\nu_\theta}$ numerically and measure the residual \emph{forward-backward error}
\begin{align}
    \int_\Omega |(\phi_{\nu_i} \circ \phi_{-\nu_i}- \id) (x)| dx.
\end{align}
Similarly, the backward-forward error can be computed from $\phi_{-\nu_i} \circ \phi_{\nu_i}$. If the discretization is perfectly self-consistent, both residuals should be close to zero.

Fig.~\ref{fig:scaling-squaring Convergence} shows the results depending on the number of integration steps $n$. Up to $n=10$, the relative error decreases exponentially before stagnating, which we assume is caused by an accumulation of interpolation errors. For our common choice $n=7$, the average forward-backward error is below $0.05$ voxels for both the classical flow equation and our proposed matrix-valued approach on $\SE(3)$.

Perfect invertibility is neither expected nor required by the overall approach: After all, the primary goal of employing the flow equation is to formulate an expressive mapping from network outputs to a deformation field. However, the good results suggest that the scheme could be useful for a bidirectional registration approach in which both $\phi$ and $-\phi$ occur in the objective. We investigate this further in the following section. 

\subsection{Similarity term and regularizer}\label{sec:objective}
Regardless of the parametrization of the deformation, the approach~\eqref{eq:reg-param-main} still requires choosing a distance $J$ for the data term as well as a regularizer $R$.

\mypar{Similarity} We consider the global negative normalized cross-correlation $\NCC$  for the data term. In order to compare the deformed template $T$ and reference images $R$ using $\NCC$, we define the mean value $\bar{T}$ as
\begin{align}
    \bar{T}= \frac{1}{|\Omega|} \int_{\Omega} T(x) dx
\end{align}
and $\bar{R}$ analogously. The data term is then given by
\begin{align}
    J(R,T) := 1 - \NCC (R,T) := 1-\frac{\langle T-\bar{T},R-\bar{R} \rangle_{L^2}}{\sqrt{\|T-\bar{T}\|_{L^2}^2 \|R-\bar{R}\|_{L^2}^2+\epsilon}} \in [0,2] .
\end{align}
This corresponds to the negative 
cosine of the angle between the vectorized zero-mean images, and is an established multi-modal similarity term due to its invariance to affine intensity transformations~\cite{lewis1995fast}. We stabilize the division by adding $\epsilon=10^{-5}$ to the square root in the denominator.

\mypar{Regularizer}
To prevent folding caused by the discretization and/or interpolation, we include a regularization term that penalizes Jacobians with determinant smaller than a threshold $\epsilon > 0$ via \cite{mok2020fast}:
\begin{align}
     R_{\epsilon}(\phi):=\int_\Omega \max[-\det(D\phi(x))+\epsilon,0] dx.
\end{align}
In our experiments, we set $\epsilon=0.01$ in order to decrease the likelihood of the first-order optimization method accidentally passing into a region of zero or negative determinant due to the non-infinitesimal step size.

The determinants can be evaluated in the discretized setting on a (transformed) grid by dividing each voxel into five 3-simplices/tetrahedra (Fig.~\ref{fig:simplex-decomposition}) and computing their oriented volume. The latter are then summed with weights of $\frac{1}{6}$ for the outer, and $\frac{1}{3}$ for the inner tetrahedra. This approach is a more isotropic variant of the one used in~\cite{haber2004numerical}, where all tetrahedra are extended from a single vertex.

\tikzset{every picture/.style={line width=0.75pt}} 
\begin{figure}
    \centering
    \begin{tikzpicture}
		[cube/.style={very thick,black},
			grid/.style={very thin,gray},
			axis/.style={->,blue,thick}]

	\draw[cube] (0,2,0) -- (2,2,0); 
        \draw[cube] (2,2,0) -- (2,0,0);
        \draw[dashed] (0,0,0) -- (0,2,0);
        \draw[dashed] (2,0,0) -- (0,0,0);
	\draw[cube] (0,0,2) -- (0,2,2) -- (2,2,2) -- (2,0,2) -- cycle;
	
	\draw[dotted] (0,0,0) -- (0,0,2);
	\draw[cube] (0,2,0) -- (0,2,2);
	\draw[cube] (2,0,0) -- (2,0,2);
	\draw[cube] (2,2,0) -- (2,2,2);

	\draw[dotted] (0,0,2) -- (2,2,2);
	\draw[dotted] (0,0,2) -- (2,0,0);
    \draw[dotted] (0,0,2) -- (0,2,0);
	\draw[dotted] (0,2,0) -- (2,2,2);
	\draw[dotted] (2,2,2) -- (2,0,0);
    \draw[dotted] (2,0,0) -- (0,2,0);
	
\end{tikzpicture}
    \caption{Decomposition of a voxel into one inner and four outer tetrahedra. The edges of the tetrahedra are drawn with dotted lines. Evaluating the oriented volume of all tetrahedra after deformation allows to detect foldings.}
    \label{fig:simplex-decomposition}
\end{figure}
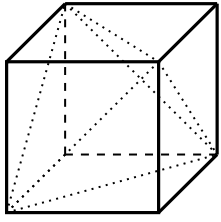

In addition to that, we also employ 
gradient or Hessian regularisation to encourage smoothness. The respective measures are defined as
\begin{align}
R_{g}(\phi):=\int_\Omega \sum_{i=1}^n\|\nabla \phi_i\|^2 dx, \quad R_{h}(\phi):=\int_\Omega \sum_{i=1}^n\|D^2 \phi_i\|^2_F dx,
\end{align}
where $\|\cdot\|_F$ denotes the Frobenius norm, evaluated at each component of the deformation vector field $\phi: \Omega \to \R^n$. Translations lie in the kernel of the sublinear functional~$R_g$ and are not penalized. Analogously, affine transformations are not penalized by~$R_h$.

\mypar{Bidirectional approach}
The simplicity of obtaining both $\phi$ and $\phi^{-1}$ suggests a \emph{bidirectional} approach: Instead of solving only
\begin{align}
    \min_{\theta \in \Rn} \LL(\phi_\theta; I_1, I_2) \quad\text{with}\quad \phi_\theta = S(G(\theta)),
\end{align}
in the bidirectional setting, we minimize
\begin{align}
    \min_{\theta \in \Rn} \half (\LL(\phi_\theta; I_1, I_2) + \LL(\phi_\theta^{-1}; I_2, I_1)) , \quad\text{where}\quad \phi_\theta^{-1} := S(-G(\theta)).\label{eq: bidirectional}
\end{align}
An appealing feature of this approach is that the loss is invariant with respect to the order of $I_1$ and $I_2$, when the deformation is inverted analogously.

\subsection{Network architecture and optimization}\label{sec:network-architecture}
\mypar{Network architecture} The network architecture for generating the velocity field $\nu_\theta$ is visualized in Fig.~\ref{subfig:experiments-2d-architecture-vae}.
The cuboid image domain is uniformly scaled to fit into the unit cube, $\Omega \subset [-1,1]^3$.
The network accepts coordinates as input and returns a right invariant vector field on $G$ from the Lie algebra $\fg$, which is a $k$-dimensional vector space. 
The batched output forms a four-dimensional tensor with dimension $n_h \times n_w \times n_z \times k$. 
Following \cite{sitzmann2020implicit}, the neural field is designed as a (fully-connected) multi-layer perceptron with sinusoidal activation functions. We modified the proposed initialization scheme by initializing the weights of the last linear layer uniformly between $[-10^{-4}, 10^{-4}]$. This ensures that the initial deformation is close to the identity transformation.

The network consists of five layers of 512 neurons each. Additionally, residual connections were added between layers 1 to 5. The neural velocity field is evaluated on an equidistant, rectangular grid to perform its subsequent integration.
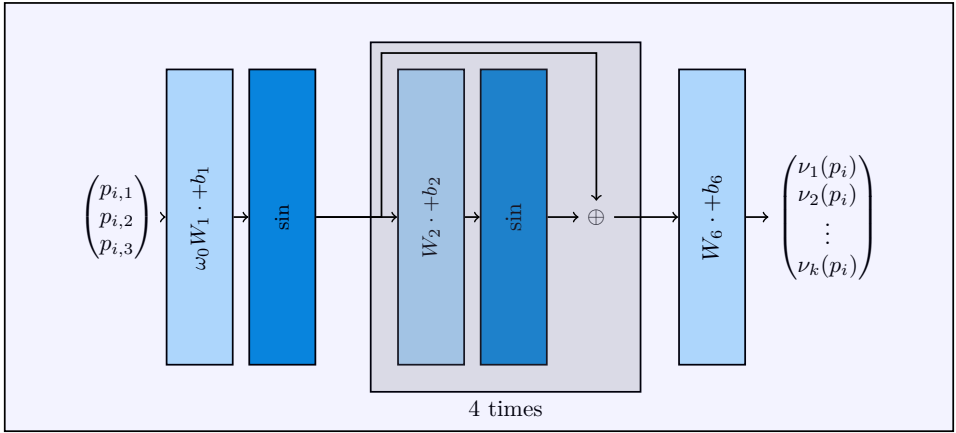
\begin{figure}[tb]
  \centering
  \hspace*{-0.75cm}
  \scalebox{0.87}{
  \begin{tikzpicture}
    \node (x) at (0,0) {$\begin{pmatrix}
        p_{i,1} \\
        p_{i,2} \\
        p_{i,3}
    \end{pmatrix}$};
 
    \node[conv,rotate=90,minimum width=4.5cm] (W1) at (1.25,0) {\small $\omega_0 W_1 \cdot + b_1$};
    \node[pool,rotate=90,minimum width=4.5cm] (sin1) at (2.5,0) {\small$\sin$};
    \node[conv,rotate=90,minimum width=4.5cm] (conv2) at (4.75,0)  {\small $W_2 \cdot + b_2$};
    \node[pool,rotate=90,minimum width=4.5cm] (pool2) at (6,0)  {$\sin$};
    \node (conv3) at (7.25,0) {$\oplus$};
    \node[conv,rotate=90,minimum width=4.5cm] (conv4) at (9,0)  {$W_6 \cdot + b_6$};
    \node (y) at (10.75,0) {$\begin{pmatrix}
        \nu_1(p_{i}) \\
        \nu_2(p_{i}) \\
        \vdots \\
        \nu_k(p_{i})
    \end{pmatrix}$};
    \node[draw,fill=gray, fill opacity=0.2, inner sep=4mm, label=below:4 times,fit=(conv2) (pool2) (conv3)] {};
    \node[draw,fill=blue, fill opacity=0.05, inner sep=10mm, fit=(x) (pool2) (y)] {};
    \draw[->] (x) -- (W1);
    \draw[->] (W1) -- (sin1);
    \draw[->] (sin1) -- (conv2);
    \draw[->] (conv2) -- (pool2);
    \draw[->] (pool2) -- (conv3);
    \draw[->] (sin1) -- (4 cm, 0) -- ++(0,2.5 cm) -| (conv3); 
    \draw[->] (conv3) -- (conv4); 
    \draw[->] (conv4) -- (y); 
    
  \end{tikzpicture}}
  \vskip 6px
  \caption{Network architecture of $G$ for generating the velocity field. The matrix-valued velocity field is described by a neural implicit representation that subsequently drives the evolution of a matrix field, generating a deformation.}
  \label{subfig:experiments-2d-architecture-vae}
\end{figure}

Overall, the network contains between $1.31\cdot 10^7$ and $1.33\cdot 10^7$ trainable weights, depending on the dimension of $\mathfrak{g}$.

\mypar{Low-frequency bias}
Studies such as \cite{rahaman2019spectral} show that neural networks exhibit a low-frequencies bias. To mitigate this effect, we scale the weights of the first layer with a preset scalar hyperparameter $w_0$ as is done in \cite{sitzmann2020implicit}:
\begin{align}
    h_j^1=\sin\left(\sum_{i=1}^3 w_0 \omega^1_{ji} \cdot x_i + b^1_j\right) .
\end{align} 
Here, the variable $x_i$ represents the $i$-th image coordinate, $h^1_j$ denotes the value of the $j$-th hidden neuron in the first layer, while $\omega^1$ and $b^1$ are the weights and biases of the first layer, respectively.

The scaling increases the frequencies of the sinusoids in the first layer during initialization which in return increases the curvature of the loss with respect to the corresponding weight parameters
For $w_0=0$, the network output does not depend on the input coordinates, but only on the biases. This produces channel-wise constant fields, yielding a translation for the SVF and a rigid deformation for the $\SE(\R,3)$ setting. Smaller values of $w_0$ generally yield smoother deformations, whereas larger values allow the reconstruction of finer deformations.

\mypar{Optimization and post-scaling} For the optimization, we used ADAM \cite{kingma2014adam} in full-batched mode, so there was no stochastic uncertainty introduced by sampling. Nonetheless, the method is not fully deterministic, as the network weights are initialized randomly and a stochastic gradient estimation is performed for the interpolation while integrating the velocity fields. As we found that the network tends to yield relatively large initial velocity fields, we introduced a scalar hyperparameter for post-scaling the field from the output of the network. In our experiments, the factor was tuned for the specific application and always lied between $10^{-4}$ and $10^{-1}]$. 
In the $\SE(3)$ and $\SIM(3)$ settings, we used two different factors, one for the rotational part (angular velocities) and one for the translational part, to accommodate their inherently different range.

Hyperparameter tuning was conducted on the learning rate (lr), the post-scaling factor (sf), and $w_0$. Using \cite{optuna_2019}, for each experiment and chosen matrix group, a tree-structured Parzen estimator was utilized and evaluated for 120 trials, each consisting of the registration of five randomly chosen volume pairs. The choice of subsequent evaluations was controlled by the estimation of the Parzen estimator. In the initial search space, $w_0$ was uniformly distributed over the interval $[0,30]$, whereas the scaling factor and learning rate were both logarithmically distributed over $[10^{-5},10^{-1}]$. Results of the tuning process are summarized in Appendix~\ref{sec:hyperparameter-tuning}.

\section{Numerical Results}\label{sec4}

To validate the proposed method and evaluate its performance, we conducted numerical experiments both on synthetic and real-world data.

The implementation for the stationary velocity fields is based on \cite{han2023diffeomorphic}, but with differently chosen hyperparameters for the learning rate, scaling factor after integration, and frequency parameter $\omega_0$, which yielded better results in our implementation. Hyperparameters used in the experiments are listed in Appendix \ref{sec:hyperparameter-tuning}.

\subsection{Datasets}\label{sec:results:datasets}
\mypar{Real-world data}
We benchmarked on the OASIS-1 dataset~\cite{marcus2007open}, which consists of $414$ brain MRI scans of individual patients in different stages of dementia and a healthy control group.
We used 5 pairs for hyperparameter tuning and the rest for evaluation. 
The image dimensions are $160\times 192 \times 224$ voxels. Besides the intensity values, each dataset contains segmentations on $35$ regions, conducted by experts. We used the skull-stripped, affinely pre-aligned volumes provided in the dataset. 

\mypar{Synthetic data}
To validate our method on data with a known ground truth deformation field, which is typically unavailable for real-world data, we generated an artificial dataset based on random non-linear deformation fields.

As our method especially aims at recovering deformations with a large rigid motion, we chose $7^3=343$ equidistantly distributed control points in the image domain $\Omega \subset \mathbb{R}^3$ and applied a rotation around a random axis to these points. Following \cite{Muller}, the axis was sampled randomly from a uniform distribution over the unit sphere; the angle was drawn uniformly distributed over the interval $[-\frac{\pi}{4}, \frac{\pi}{4}]$, producing rotations up to $45$ degrees. The rotation was then augmented by a random translation with its length chosen uniformly distributed over $[0,0.1]$.

To this global rigid deformation, we then added a non-rigid component by perturbing each control point by a random translation with length uniformly distributed over $[0,0.05]$. The dense synthetic deformation vector field $\phi_{syn}$ was ultimately computed by interpolating between the perturbed control points using radial basis functions based on second-order polyharmonic splines~\cite{forti2014efficient}.

To generate the image pairs, the deformation was applied to a zero-padded reference image from the OASIS dataset, to which Gaussian noise with mean zero and standard derivation $0.01$ was added, resulting in a synthetic template image. Some exemplary deformations can be seen in Fig.~\ref{fig:example-defs}.
\begin{figure}[tbp]
    \centering
    \scalebox{0.9}{
    \begin{tabular}{cccc}
      \rotatebox{90}{\quad\quad\quad\quad coronal}&\includegraphics[width=0.3\textwidth]{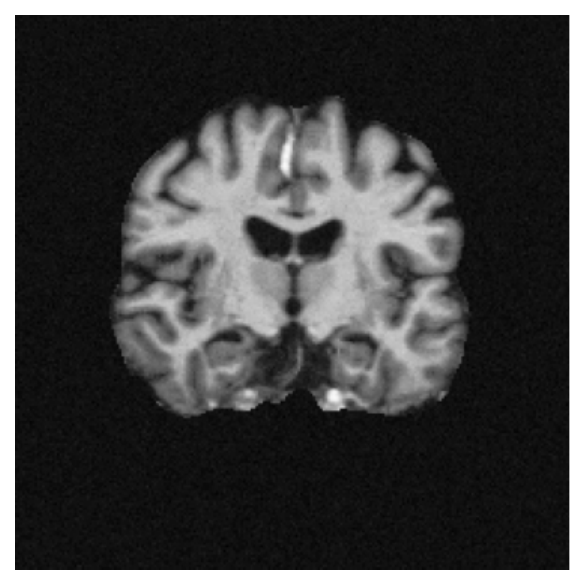} &
      \includegraphics[width=0.3\textwidth]{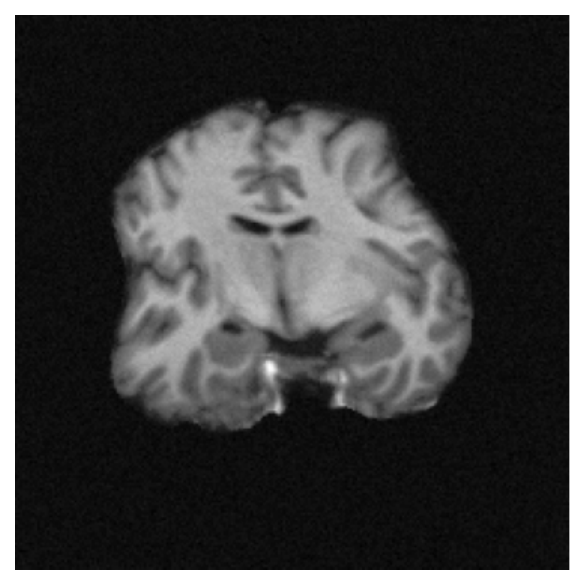} &
      \includegraphics[width=0.3\textwidth]{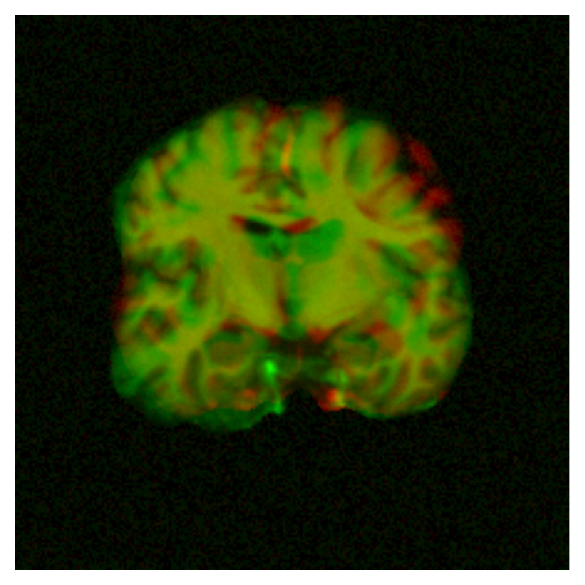} \\
        \rotatebox{90}{\quad\quad\quad\quad axial} & \includegraphics[width=0.3\textwidth]{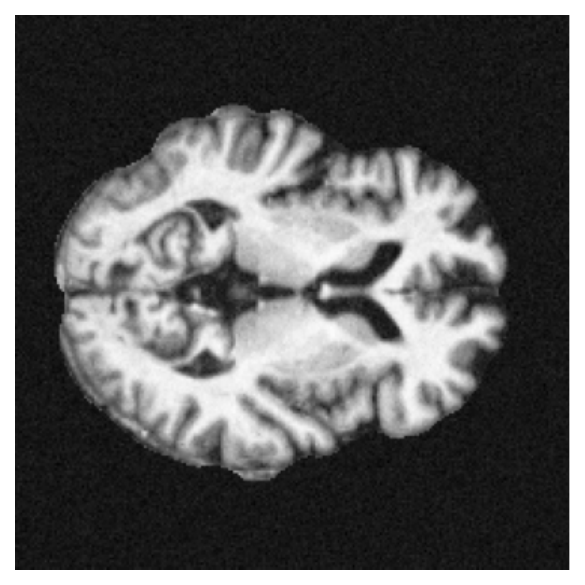} &
        \includegraphics[width=0.3\textwidth]{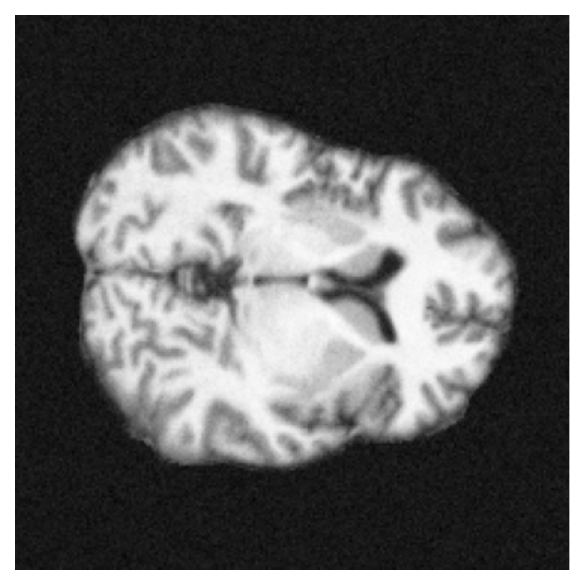} &
        \includegraphics[width=0.3\textwidth]{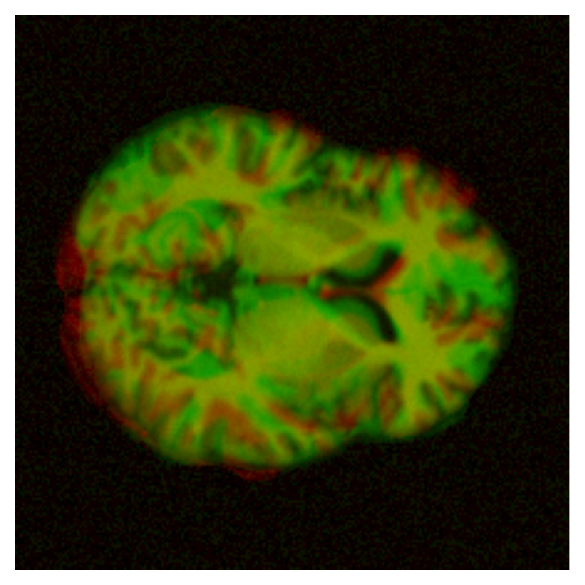}\\
        \rotatebox{90}{\quad\quad\quad\quad central} & \includegraphics[width=0.3\textwidth]{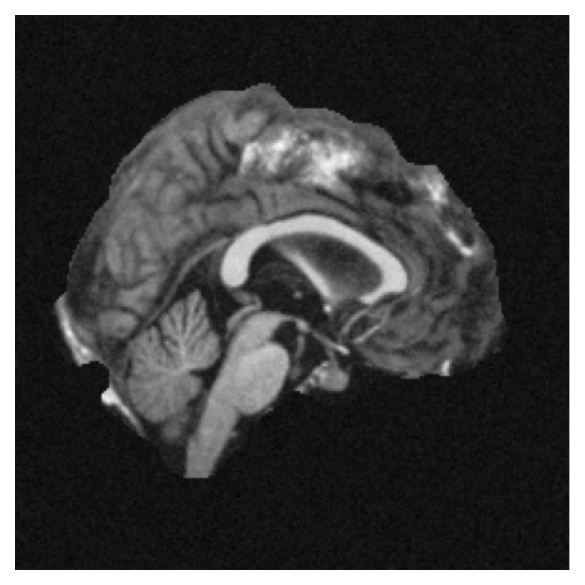} &
        \includegraphics[width=0.3\textwidth]{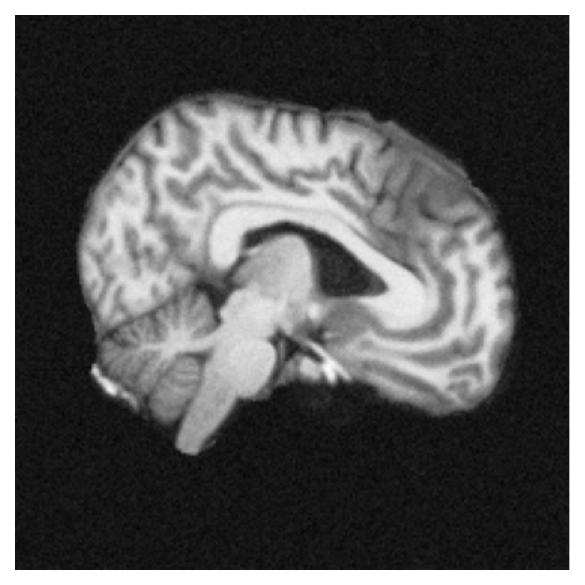} &
        \includegraphics[width=0.3\textwidth]{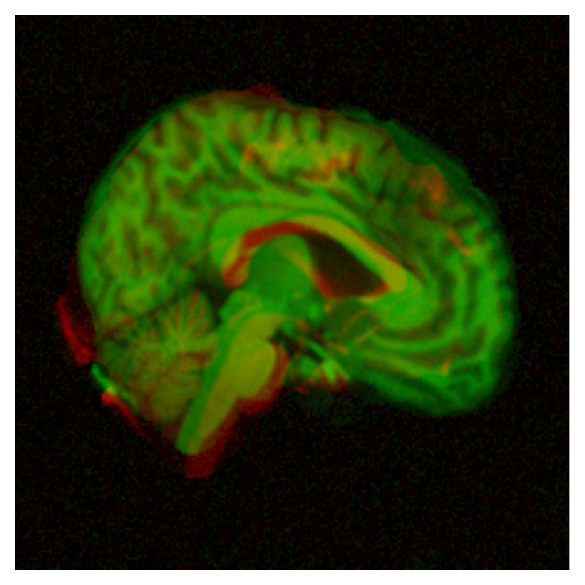} \\
        & reference & generated deformation & overlay
        \end{tabular}}
        \vspace{0.5em}
    \caption{Exemplary slices from the 3D OASIS benchmark dataset (left). In order to accurately evaluate the full registration process, we generated synthetic ground truth deformations and deformed the images (center). The color overlay (right) highlights the differences (reference image in red, deformed image in green). While this example exhibits only a small roto-translational component, the full synthetic dataset contains rotations up to 45 degrees. 
    }
    \label{fig:example-defs}
\end{figure}


\subsection{Benchmark Metrics}\label{sec:benchmark-metrics}
For the evaluation of the results, we used different metrics:
\begin{itemize}
\item \textbf{Root-mean-square error.}
For the synthetic test data with known ground truth, we computed the root-mean-square error (RMSE) between the recovered deformation field and the ground truth deformation field, measured in pixels: 
\begin{align}
    \RMSE_{\Omega_M}(\phi_{syn}, \phi_{reg}) =  \sqrt{\frac{1}{|\Omega_M|}\int_{\Omega_M} |\phi_{syn}-\phi_{reg}|^2} 
\end{align}
In this, we masked the background of the reference image by design of the domain $\Omega_M$ so that the deformation field was compared to the ground truth only at points where image information was available.

\item \textbf{Dice score.}
For the real-world data, no dense ground truth was available. To still get a coarse indication of the registration quality, we followed an approach common in medical image registration when segmentation annotations are available. Given two segmented regions (i.e., subsets of $\Omega$) $S_1$ and $S_2$ on the first and second image and a computed deformation $\phi$, we deformed $S_2$ using $\phi$ and computed the distance to $S_1$ using the (forward) Dice score
\begin{align}
    \DS(S_1, \phi^{-1}(S_2)), \label{eq:dice-score}
\end{align} where
\begin{align}
    \DS(A,B)=\frac{2|A\cap B|}{|A|+|B|}.
\end{align}

As this metric depends on the transformation direction, we also list the backward Dice score $\DS(\phi(S_1),S_2)$. 
\item \textbf{SSIM.}
As an additional metric, we computed the Structural Similarity
\begin{align}
    \SSIM(I_1, I_2 \circ \phi)
\end{align}
between reference and deformed template image, as this is closer to human perception than simpler norm-based metrics~\cite{wang2004image}.
Note that this metric is purely based on intensity values and that there are generally many deformations producing the same deformed image.

\end{itemize}
Lastly, we counted the percentage of sub-pixel simplices of the deformed grid at which the diffeomorphism property $\det D \phi (x) < 0$ was violated; for details, see Section~\ref{sec:objective}.

\subsection{Results}

All benchmarks were implemented in PyTorch 2.3.0 and performed on a 24-core AMD EPYC 74F3 system with 256GB of RAM, 3x NVIDIA A100 and CUDA 12.0.

\mypar{Fitting deformations} We first validated the effectiveness of the proposed approach for expressing arbitrary deformations with large rigid components by reconstructing synthetic deformations generated as described in Section~\ref{sec:results:datasets}. As optimization loss, the squared RMSE was used.
Note that despite being convex in the resultant deformation, the loss term and, thereby, the optimization is non-convex in the network weights $\theta$.

The results in Fig.~\ref{fig:deformation-fitting}, evaluated for 409 synthetically generated deformations, show that the matrix field-based approaches both using $\SE(3)$ and $\SIM(3)$ allow to approximate the deformation field better than with a classical SVF approach. As initial error, we measured the difference between the synthetic deformation and the identity mapping.

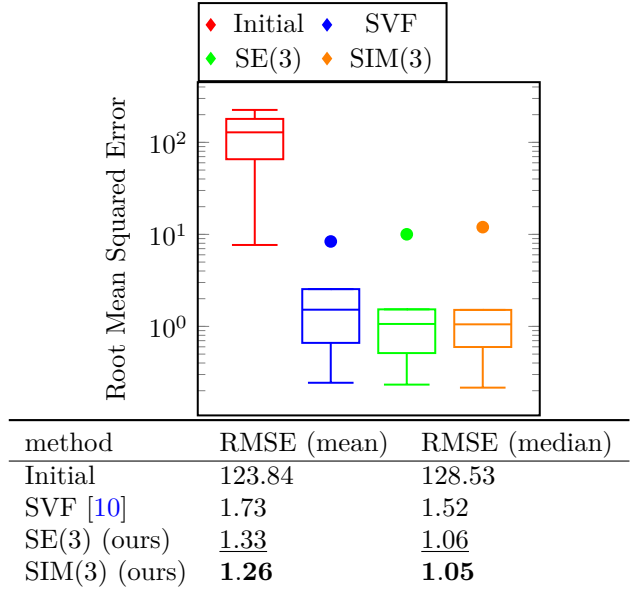
\begin{figure}[t]
  \centering
  \begin{tikzpicture}
  \begin{semilogyaxis}[
  boxplot/draw direction=y,
  height=6cm,
  ylabel={Root Mean Squared Error},
  boxplot={
      %
      draw position={1/5 + floor(\plotnumofactualtype/4) + 1/5*mod(\plotnumofactualtype,4)},
      box extend=0.15,
  },
  x=5cm,
  xtick={0,1,2,...,5},
  x tick label as interval,
  xticklabels={%
      {Velocityfield Representation\\},%
  },
  x tick label style={
      text width=2.5cm,
      align=center
  },
  cycle list={{red},{blue},{green},{orange}},
  custom legend,
            legend style={at={(0,1)},anchor=south west,legend columns=2,
                column sep=0.5em}
]    
    \addplot
    table[row sep=\\,y index=0] {
    data\\
    128.526\\
    65.593\\
    179.858\\
    7.6734\\
    225.170\\
    }; 
    \addlegendentry{Initial}
    \addplot
    table[row sep=\\,y index=0] {
    data\\
    1.519\\
    0.662\\
    2.542\\
    0.244\\
    8.375\\
    }; 
    \addlegendentry{SVF}
    \addplot
    table[row sep=\\,y index=0] {
    data\\
    1.063\\
    0.513\\
    1.532\\
    0.233\\
    10.021\\
    };
    \addlegendentry{$\SE(3)$}
    \addplot
    table[row sep=\\,y index=0] {
    data\\
    1.053\\
    0.597\\
    1.511\\
    0.216\\
    11.986\\
    };
    \addlegendentry{$\SIM(3)$}
  \end{semilogyaxis}
  \end{tikzpicture}
  \hspace{2cm}
    \begin{tabular}{lll}
    \toprule
     method & RMSE (mean) & RMSE (median)\\
     \hline
     Initial & 123.84 & 128.53 \\
     SVF \cite{han2023diffeomorphic} &   1.73 & 1.52 \\
     SE(3) (ours) & \underline{1.33} & \underline{1.06} \\
     SIM(3) (ours) & $\mathbf{1.26}$ & $\mathbf{1.05}$ \\
     \bottomrule
    \end{tabular} 
    \caption{
    RMSE of fitting large synthetic deformations using the classical SVF approach and the proposed approach with two different matrix fields $\SE(3)$ and $\SIM(3)$. The method is evaluated for $409$ random deformations, and the RMSE is computed after $120$ iterations each. The proposed parametrization aids the optimization process in finding a more precise representation of the deformation, resulting in a lower RMSE. Best scores are typeset in bold, second-best scores are underlined.}
    \label{fig:deformation-fitting}
\end{figure}

\mypar{Registering synthetic deformations} 
Figure~\ref{fig:self-warped} shows the results for the pairwise image registration on a set of 409 image pairs with synthetic deformations. While the SVF approach greatly struggled to reconstruct the deformations here, presumably due to their large rigident, the proposed use of matrix fields allowed us to robustly approximate the correct global deformation.

\begin{figure}
  \centering
  \begin{tikzpicture}
  \begin{semilogyaxis}[
  boxplot/draw direction=y,
  ylabel={Root Mean Squared Error},
  height=6cm,
  boxplot={
      %
      draw position={1/5 + floor(\plotnumofactualtype/4) + 1/5*mod(\plotnumofactualtype,4)},
      box extend=0.15,
  },
  x=5cm,
  xtick={0,1,2,...,5},
  x tick label as interval,
  xticklabels={%
      {Velocityfield Representation\\},%
  },
  x tick label style={
      text width=2.5cm,
      align=center
  },
  cycle list={{red},{blue},{green},{orange}},
  custom legend,
            legend style={at={(0,1)},anchor=south west,legend columns=2,
                column sep=0.5em}
]    
    \addplot
    table[row sep=\\,y index=0] {
    data\\
    45.43280553817749\\
    24.82961869\\
    60.60968828\\
    2.557808667421341\\
    74.50289344787598\\
    }; 
    \addlegendentry{Initial}
    \addplot
    table[row sep=\\,y index=0] {
    data\\
    45.63654461247455\\
    24.8319677\\
    60.75529373\\
    2.630602305629001\\
    78.77947816120819\\
    }; 
    \addlegendentry{SVF}
    \addplot
    table[row sep=\\,y index=0] {
    data\\
    3.0103998910658425\\
    2.84435738\\
    3.33117707\\
    2.5981549975224487\\
    65.32140795321553\\
    };
    \addlegendentry{$\SE(3)$}
    \addplot
    table[row sep=\\,y index=0] {
    data\\
    3.0718562758292727\\
    2.9131065\\
    3.49729799\\
    2.6057837298517823\\
    51.32890446651967\\
    };
    \addlegendentry{$\SIM(3)$}
    \end{semilogyaxis}
    \end{tikzpicture}
    \hspace{2cm}
    {\small
    \begin{tabular}{lll}
    \toprule
     method & RMSE (mean) & RMSE (median) \\
     \hline
     Initial & 43.35 & 45.43 \\
     SVF \cite{han2023diffeomorphic} &  43.47 & 45.64 \\
     SE(3) (ours) & \underline{6.25} & $\mathbf{3.01}$ \\
     SIM(3) (ours) &  $\mathbf{4.87}$ & $\underline{3.07}$ \\
     \bottomrule
    \end{tabular}}
    \caption{Registration of synthetically generated deformations using the classical SVF approach and the proposed approach with two different matrix fields $\SE(3)$ and $\SIM(3)$. The RMSE in the deformation after $120$ iterations is shown, computed for $409$ random deformations. In this particular setting, SVF seems to fail at capturing the rotations, while the parametrization with matrix groups produces deformations close to the ground truth in most cases.}
    \label{fig:self-warped}
\end{figure}
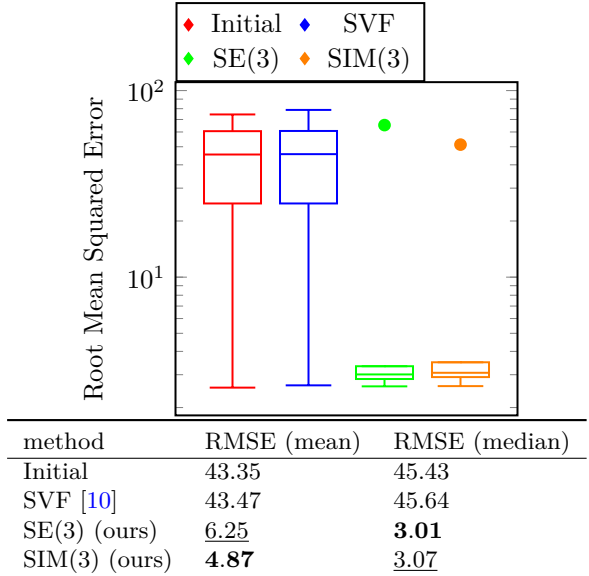

This behavior is clearly visible in Fig.~\ref{fig:alignments-iso} and Fig.~\ref{fig:displacement fields} shown in the introduction, in which the resultant images as well as an exemplary slice of the deformation field are depicted under small and large deformations. It illustrates how the matrix group approach is able to capture even large rotational deformations, whereas the use of an SVF tends to align the intensity values by local deformations instead of a global rotation, resulting in the large RMSE observed in Fig. \ref{fig:self-warped}.

To further examine at which point the SVF approach breaks down, we repeated the experiment by varying the magnitude of the synthetically generated deformations and measured the RMSE of the recovered deformations.
The experiments were performed both with and without a global rotation, and the results are presented in Fig.~\ref {fig:warping}.

\begin{figure}[tb]
    \centering
    \subfloat[without global rotation]{
    \includegraphics[width=0.49\linewidth]{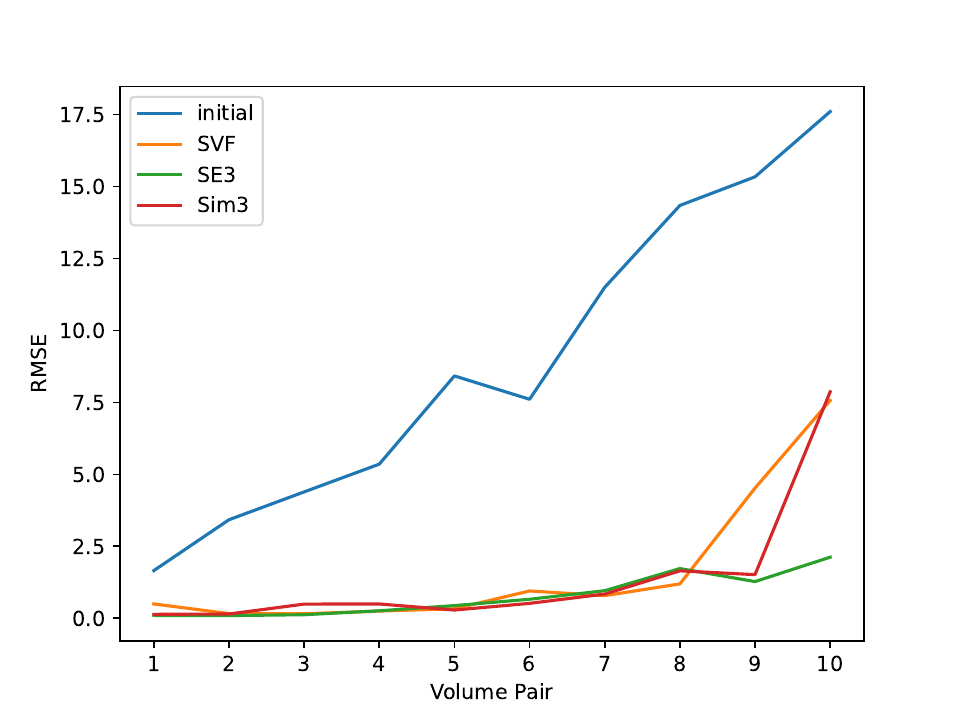}}
    \subfloat[with global rotation]{
    \includegraphics[width=0.49\linewidth]{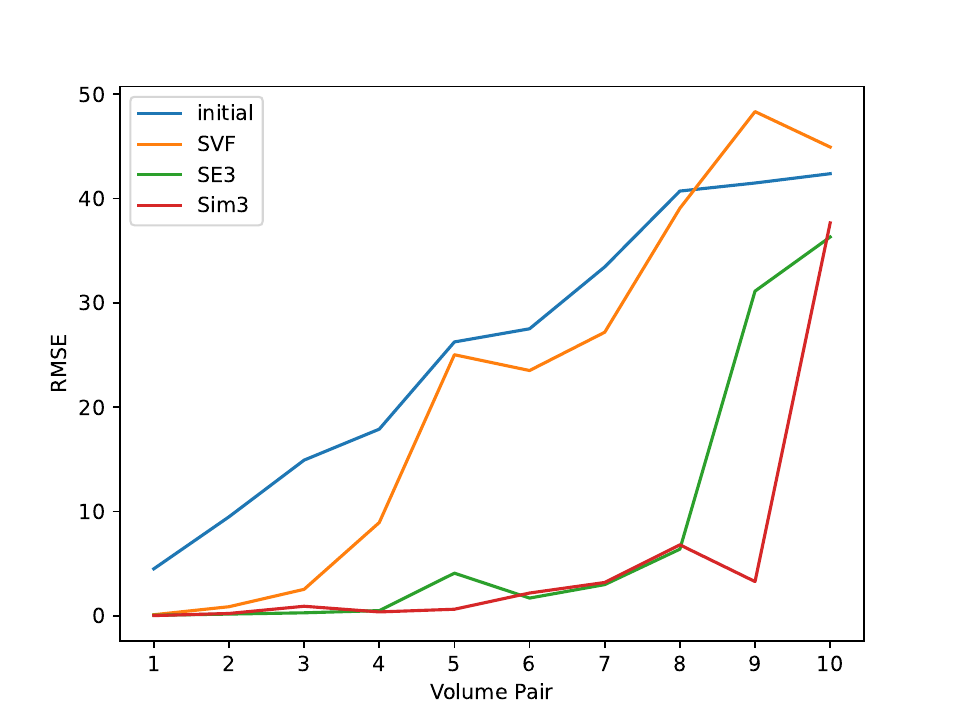}}
    \caption{Registration results of OASIS image pairs with synthetically generated nonlinear deformations of increasing magnitude. The RMSE between the recovered deformation field and the ground truth is shown.
    In the left figure, only a global translation and local perturbances are performed. The global translation increases from 0-10\% of the domain width from left to right. In the right figure, a global rotational component is further added whose angle increases from 0-90\%.
    Without a global rotation, all methods work reasonably well even for larger translations. In the presence of global rotations, however, the classical SVF breaks down already at smaller angles, whereas even large rotations up to~70° only have a small impact on the RMSE when using the proposed parametrization with matrix-valued fields (SE(3), SIM(3)), .}
    \label{fig:warping}
\end{figure}

To this end, we generated 10 rotation-free volumes by fixing the rotation angle to zero and linearly scaling the magnitude of the global translation as well as the local non-rigid deformation, such that the global translation was scaled between 0\% (volume pair 1) and 10\% (volume pair 10) of the domain width.
We then generated 10 more volumes, in which a rotation was added, and the angle was linearly increased from 0° (volume pair 1) to 90° (volume pair 10).
All 20 volumes were generated separately, that is, with different random components.

Fig.~\ref {fig:warping}(a) shows that the SVF, as well as both matrix-based methods, performed reasonably well even under large translations. 
In the presence of additional global rotations, however, the picture changed considerably. As can be seen in Fig.~\ref{fig:warping}(b), the SVF-based approach broke down completely once the rotation angle exceeded 30°, ultimately yielding deformations that hardly had a lower RMSE than the identity. The matrix-valued approaches, in contrast, returned stable results close to the ground truth even under larger rotations, and only started deteriorating at rotations with an angle greater than 70°.

We conclude that at least in these experiments, the matrix group approach achieved the goal of improving the network's ability to find deformations with large rotational components, and clearly outperformed the approach solely based on SVFs.
In all experiments, we used NCC as similarity term and the Hessian regularisation $R_h$, scaled with factor $2\cdot 10^{-5}$.

\mypar{Registering real-world OASIS data}
Finally, we tested our approach on the task of inter-patient registration with real-world data.
We used NCC as image similarity metric, gradient regularization $R_g$ on the final displacement field, scaled with a factor of 0.1, and the folding prevention term $R_\epsilon$ with a factor of~200.
We compared the classical SVF method with the matrix field-based methods for $\SE(3)$ and $\SIM(3)$, each in forward and bidirectional mode \eqref{eq: bidirectional}. Furthermore, we included the SyN method \cite{avants2008symmetric} with setting~$(100,100,25)$ for the maximal number of optimization steps in the pyramid in the comparison. 
The bidirectional approach improves the backward Dice score and the mean Dice score in the parametrizations with an SVF and the matrix field approaches, as expected.

For the experiments, we registered $150$ randomly selected image pairs from the OASIS-1 dataset. Additionally, we benchmarked on a task with larger deformations starting from the unaligned brain scans from \cite{marcus2007open}. We used FreeSurfer \cite{spf2007} for skull stripping the unprocessed volumes.
As there are no ground truth deformation fields available, we used the Dice score as a proxy.
For the second task, we furthermore list the SSIM metric.

The results of the benchmark are presented in Table~\ref{tab:table} and Table~\ref{tab: registration OASIS}. They show that, in general, the matrix-group approach performed similarly well as the classical SVF approach, both for the pre-aligned images as well as for the unaligned brain scans. In the unaligned case, that is, for larger deformations, both approaches also improve the performance on the Dice score compared to SyN~(Table~\ref{tab: registration OASIS}).

\begin{table}[tbp]
    \centering
    \resizebox{13.2cm}{!}
{\begin{tabular}{lllll} %
    \toprule
     method & Dice score $\pm$ std.~dev. ($\uparrow$) & Dice score backwards ($\uparrow$) & mean Dice score ($\uparrow$) & $\det(J_\phi)\leq 0$ ($\downarrow$)\\
     \hline
     initial  &  0.5838 $\pm$ 0.0610 & 0.5838 $\pm$ 0.0610 & 0.5838 $\pm$ 0.0610 &  - \\
     SVF \cite{han2023diffeomorphic} &   0.8153 $\pm$ 0.0231 & 0.8002 $\pm$ 0.0233 & 0.8077 $\pm$ 0.0223 & $2.35\cdot10^{-6}$ \\
     bidir. SVF &  0.8149 $\pm$ 0.0235 & \underline{0.8160 $\pm$ 0.0210} & \underline{0.8154 $\pm$ 0.02197} &  $3.58\cdot10^{-6}$ \\
     SE(3) [ours] & $\mathbf{0.8159 \pm  0.0231}$ & 0.8001 $\pm$  0.0233 & 0.8080 $\pm$ 0.0229 & $1.85\cdot10^{-6}$\\
     bidir. SE(3) [ours] & 0.8148 $\pm$ 0.0238 & 0.8159 $\pm$ 0.0216 & 0.8153 $\pm$ 0.0224 & $2.90\cdot10^{-6}$\\
     SIM(3) [ours] & \underline{0.8157 $\pm$ 0.0233} & 0.8009 $\pm$  0.0233 & 0.8083 $\pm$ 0.0231 & $\underline{1.95\cdot10^{-6}}$\\
     bidir. SIM(3) [ours] &  0.8151 $\pm$ 0.0236 & \textbf{0.8162 $\pm$ 0.0212} & $\mathbf{0.8156 \pm 0.0221}$ & $3.20\cdot10^{-6}$ \\
     SyN \cite{avants2008symmetric} & 0.8062  $\pm$ 0.0248 & 0.8069 $\pm$ 0.0235  & 0.8066 $\pm$ 0.0241  &  $\mathbf{0}$\\
     \bottomrule
     \end{tabular}
}   
    \caption{Results on inter-patient 3D registration on pre-aligned and skull-stripped volumes of the OASIS-1 dataset. The classical (SVF) and proposed the matrix-valued approaches ($\SE(3)$ and $\SIM(3)$) all yield similar Dice scores after registration. A bidirectional approach improves the backward Dice score and the mean of both scores.
    }
    \label{tab:table}
\end{table}

\begin{table}[tbp]
    \centering
    \scalebox{1}{
    \begin{tabular}{llll} 
    \toprule
     method & Dice score, forward ($\uparrow$) & SSIM, forward ($\uparrow$) & fraction with $\det(J_\phi)\leq 0$ ($\downarrow$)\\
     \hline
     initial &  $0.1178 \pm 0.1001$ & $ 0.7987 \pm 0.011$  & - \\
     SVF \cite{han2023diffeomorphic} &  $\mathbf{0.7278 \pm 0.0628}$ & $0.9444 \pm  0.0170$ & $\mathbf{4.04 \cdot 10^{-7}}$\\
     $\SE(3)$  & $\underline{ 0.7223 \pm  0.1012}$  & 
     $\underline{0.9479 \pm  0.0158}$  & $7.65 \cdot 10^{-7}$\\
     $\SIM(3)$ & $0.6199\pm 0.2416$ & $0.9367 \pm  0.0187$ & $\underline{4.75 \cdot 10^{-7}}$\\
     SyN \cite{avants2008symmetric} &  $0.4987 \pm 0.3081$ & $ \mathbf{0.9572\pm 0.0266}$  & $4.10 \cdot 10^{-6}$ \\ 
     \hline
    \end{tabular}}
    \caption{Results on unidirectional registration of unaligned OASIS brain scans. SVF and SE(3) work comparably well in terms of Dice score, beating the SyN method by a wide margin. Note that Dice score and SSIM provide only limited insight into the quality of the actual deformation field, compare Fig~\ref{fig:displacement fields} and Fig~\ref{fig:warping}.
    }
    \label{tab: registration OASIS}
\end{table}

These results stand in contrast to the ones on the synthetic examples presented in Fig.~\ref{fig:self-warped} and Fig.~\ref{fig:warping}, in which the matrix-based approaches SE(3) and SIM(3) performed significantly better than SVF. The reason for this performance difference likely lies in the nature of the underlying deformations: For the pre-aligned images, there are only marginal global deformations sought, and even in the unaligned case, the global rotational component of the deformation is relatively small.

Furthermore, in the synthetic examples, we were able to compare the recovered and ground truth deformation fields directly. On the OASIS data set (and other real-world image registration data), no dense ground truth deformation is available, which is why we resort to comparing Dice scores and SSIM. As these metrics, however, are only based on the segments and intensity values and thereby not explicitly on the deformation itself, they ultimately provide only limited insight into the quality of the actual deformation field. Therefore, even though SVF generated good results in terms of the deformed images, it might be obscured that the deformation is far from an actually sought ``truth'', see Fig.~\ref{fig:displacement fields} and Fig~\ref{fig:warping}.

\section{Conclusion}\label{sec5}
In this paper, we proposed, discussed and analyzed a novel approach to deformable image registration which extends the concept of integrating stationary velocity fields to matrix groups. We proved the unique existence of a solution to the extended flow equation in Thm. \ref{thm:existence} and derived a scaling-and-squaring algorithm to efficiently approximate this solution, see Alg. \ref{alg:exSnQ}.

Fairly evaluating non-linear image registration methods on real data is notoriously hard, as there is typically no ground truth for the deformation field available. As we observed, popular methods such as SVF can perform well in terms of intensity value based metrics as well as proxy metrics such as the Dice score on segmentations, while the generated deformation field might be far from the ground truth or a desired deformation.

Our experiments indicate that moving velocity field and flow equation from a purely translational appraoch (as in SVF) to the matrix group setting can ameliorate this issue. This is especially the case when the ``true'' deformation field comprises larger affine deformations such as a global rotation.

The matrix-based formulation allows the specification of the velocity field in a chosen matrix group which naturally covers a wider range of deformations. In our approach, we parametrize the matrix group by implicit neural representations, allowing for optimization techniques from machine learning. In future work, it will be interesting to adapt this idea to other network architectures, such as UNet- or Transformer-based approaches.

The formulation with a matrix group gives a fairly general method to choose between different parametrization schemes.
Further paths to explore include using the affine group $\text{Aff}(3)$ as parametrization for the velocity field. The lack of a closed form for the matrix exponential and logarithm in this setting, however, would require a different interpolation strategy or the use of iterative methods, making the optimization more difficult and time-consuming.
Other suitable matrix groups might also be considered.

Another interesting direction of research concerns the extension of matrix-valued fields to approaches with a non-stationary flow such as the LDDMM framework. However, this would require a different numeric integration method than scaling-and-squaring, which was central to our approach.

Overall, we hope that the proposed matrix group approach provides a building block for creating robust registration methods that can cope with large deformations not only in terms of image similarity, but also in terms of the quality of the deformation fields.

\section{Proofs}
\subsection{Proof of Theorem~\ref{thm:existence}}\label{sec:existence-proof}

\begin{proof}
Embedding the manifold of a matrix group in a surrounding vector field space (here $\R^{4 \times 4}$ with the operator norm) allows the use of classical results of ODE theory in this setting.
For the matrix subgroup $G \subset GL(\R,4)$, the space of continuous matrix fields $C(\Omega, G)$ can be embedded in the unitary function Banach algebra 
\begin{equation}
    \mathcal{B}:=C(\Omega,\R^{4\times 4})
\end{equation}
with the corresponding pointwise supremum norm on the compact domain $\Omega \subset \R^k$,
\begin{align}
    \|B\|_{\banach} := \sup_{x \in \Omega}|B(x)|_\frobenius \quad \text{for } B \in \banach , \label{eq:bnormsup}
\end{align}
and pointwise matrix multiplication.

Crucially, this norm is sub-multiplicative with regard to the pointwise multiplication of matrix fields. 
Consider for now $M$ as a 
univariate function of time $M:[0,1] \to \banach$. 
The time-invariant case of Equation \eqref{eq2:matflow} can then be written compactly as
    \begin{align}
    \frac{\partial M}{\partial t}&=f(M)  \label{eq:compact} \\
    M(0)&=id
    \end{align}
    with
    \begin{align}
        f&: \, \mathcal{B} \to \mathcal{B},\\
        f(B)&:=\left(x \mapsto i\left(\nu(PB(x)\Bar{x})\right)_{B(x)}\right),
    \end{align}
    for the canonical embedding $i$ which extends the right invariant vector fields on $G$ to the vector fields on $\R^{4 \times 4}$, and the projection $P:\R^4\to\R^3$ which removes the last component. For the sake of readability, we will omit this embedding in the following.
    
    As $\nu: \R^3 \to \mathfrak{g}$ maps each point to a right-invariant vector field on the manifold $G$, this vector field is uniquely described by its value at the identity of the matrix group:
    \begin{align}
           \nu(P B(x) \barx)_{B(x)}=\nu_{\id}(P B(x) \barx)B(x). \label{eq: identity value} 
    \end{align}
    By sub-multiplicativity of the matrix norm and the compactness of $\Omega$ (by which $\nu_{\id})$ is bounded on $\Omega$), it follows that
    \begin{align}
        \|f(B)\|_\banach = \sup_x |\nu_{\id}(P B(x) \barx)B(x)| \leq \sup_x |\nu_{\id}(P B(x)\barx)| \sup_x |B(x)|= C_\nu \|B\|_\banach ,
        \label{eq:lin_bound}
    \end{align} 
    yielding that $f$ is linearly bounded.
    By Gronvall's inequality \cite{schmidt1974w} it holds
    \begin{align}
        \|M(t)\|_B \leq \|M(0)\|_B e^{C_\nu t} = \|id\|_Be^{C_\nu t}, \label{eq: gronvall}
    \end{align}
    hence, every solution $M(t)$ is bounded for every finite time interval.

    Next, we show that $f$ is locally Lipschitz continuous. For $S, T \in \banach$, it holds
    \begin{align}
        \|f(S)-f(T)\|_\banach &= \sup_{x \in \Omega} |\nu(P S(x)\Bar{x})_{S(x)}-\nu(P T(x)\Bar{x})_{T(x)}|_\frobenius.
    \end{align}
    Using \eqref{eq: identity value} and adding zero, we obtain
    \begin{align}
        &= \sup_{x \in \Omega} |\left(\nu_\id(P S(x)x)-\nu_\id(P T(x)\Bar{x})\right)S(x)-\nu_\id(P T(x)\Bar{x})(T(x)-S(x))|_\frobenius.
    \end{align}
    By the triangle inequality and the boundedness of $\nu_{\id}$, we can estimate
    \begin{align}
        &\leq \sup_{x \in \Omega} |\left(\nu_\id(P S(x)\Bar{x})-\nu_\id(P T(x)\Bar{x})\right)S(x)|_\frobenius + \sup_{x \in \Omega} C_\nu |T(x)-S(x)|_\frobenius.
    \end{align}
    Using sub-multiplicativity,
    \begin{align}
        &\leq \sup_{x \in \Omega} |\nu_{\id}(P S(x)\Bar{x})-\nu_{\id} (P T(x)\Bar{x})|_\frobenius\sup_{x' \in \Omega}|S(x')|_\frobenius + C_\nu\|T-S\|_{\banach},
    \end{align}
    and then Lipschitz continuity of $\nu_{\id}$, we arrive at
    \begin{align}
        &\leq  L_{\nu} \sup_{x \in \Omega} |PS(x)\Bar{x}-PT(x)\Bar{x}|\,\|S\|_\banach+ C_\nu\|T-S\|_{\banach}.
    \end{align}
    Finally, using sub-multiplicativity, we conclude
    \begin{align}
        \|f(S)-f(T)\|_\banach&\leq  L_{\nu} \|S-T\|_\banach \, \sup_{x \in \Omega} |\Bar{x}| \, \|S\|_\banach+ C_\nu\|T-S\|_{\banach}\\
        &= (L_{\nu} \sup_{x \in \Omega} |\bar{x}| \, \|S\|_\banach + C_{\nu})\|T-S\|_{\banach}. \label{eq: lipschitz}
    \end{align}
    Hence, as $M(t)$ is bounded for a bounded time interval, $f$ is Lipschitz on some open neighborhood of the image of $M$ over this time.
    By Picard's theorem for Banach space-valued functions \cite{brezis2011functional}, there exists a unique solution $M:[0,1] \to \banach$ of Equation \eqref{eq2:matflow}.
    In fact, as the vector field $\nu(x$ is tangential on $G$ everywhere, it holds, that $M:[0,1] \to C(\Omega, G)$.
    By~\eqref{eq: gronvall}, this solution remains finite for finite time, in particular also on the unit interval $[0,1]$, which is important for the extended flow equation.
\end{proof}

\subsection{Proof of Theorem~\ref{thm:decomp}}\label{sec:decomp-proof}
\begin{proof}
    The proof consists of four steps. Throughout the proof, we will associate $M(t)(x)$ with $M(x,t)$:
    \begin{enumerate}
        \item  We first show that $M\in \mathcal{C}^2([0,1],\banach)$.
        \item We then prove that a semi-discrete forward Euler scheme in time, formulated in the embedding space $\banach$, converges to the solution of the flow equation~\eqref{eq2:matflow}.
        \item We afterwards construct an exponential discretization scheme on the manifold and show convergence towards the Euler scheme. This discretization satisfies the decomposition condition by construction. This scheme also builds the foundation of the scaling-and-squaring approach of Alg.~\ref{alg:exSnQ}.
        \item Finally, we prove that the decomposition condition is preserved under convergence and therefore transfers to the fully-continuous solution of the flow equation.
    \end{enumerate}

\mypar{Calculating $M''$}
    We aim at calculating the second derivative $M''$ of the solution of the extended flow equation with respect to the time.
    To this end, we first calculate the derivative of
    \begin{align}
        f_{\id} : \ &\banach \to \banach,\\
        &B \mapsto \nu_{\id}(P B(\cdot)\bar{\cdot}) \in \banach .
    \end{align} We denote the normed function space of continuous vector-valued functions as $C(\Omega,\R^3)$ with the corresponding supremum norm and apply the decomposition $f_{\id}= f_1 \circ f_2$ with $f_2: \banach \to  C(\Omega,\R^3)$, $B \mapsto P B(\cdot)\bar{\cdot}$ and $f_1: C(\Omega,\R^3) \to \banach$, $V \mapsto \nu_\id( V(\cdot))$.    
    We show that the Fréchet derivative of $f_1$ at $V$ is the left-multiplication
    \begin{align}
    D f_1(V)(h) = (D\nu_\id)(V(\cdot)) h(\cdot). \label{eq:frechetderiv}
    \end{align}
    This follows from
    \begin{align}
        &\sup_{x \in \Omega} |\nu_\id \left ( \left( V+h \right)(x)\right ) -\nu_\id(V(x)) - (D\nu_\id)(V(x)) h(x)|_\frobenius \\
        =  &\sup_{x \in \Omega} |\nu_\id \left( V(x) \right) + (D\nu_\id)(V(x)) h(x) + o(|h(x)|) - \nu_\id(V(x)) - (D\nu_\id)(V(x)) h(x)|_\frobenius\\
        = &\sup_{x \in \Omega} |o(|h(x)|)|. \label{eq:proofthm2:o-term}
    \end{align}
    

    As $\nu_\id \in C^1(\Omega,\R^{4\times4})$, it has a uniformly continuous derivative on the compact domain $\Omega$ and is therefore uniformly differentiable on $\Omega$.
    By definition, this means that~\cite{bartle2000introduction}
    \begin{align}
        \forall \epsilon>0 \, &\exists \delta>0: \, \forall x \in \Omega \ \exists L_x \in L(\R^3,\R^{4 \times 4}):  \nonumber\\
        &\forall h: 0 < |h| \leq \delta :   \frac{|\nu_\id (x+h)- \nu_\id (x) - L_x h|}{|h|} \leq \epsilon.
    \end{align}

     Hence, the convergence of the $o$-term in Eq. \eqref{eq:proofthm2:o-term} holds uniformly for all~$x\in\Omega$, and we can conclude
    \begin{align}
        \|\nu_\id \left ( (V+h)(\cdot) \right ) -\nu_\id(V(\cdot)) - (D\nu_\id)(V(\cdot)) h(\cdot)\|_\banach = o(\|h\|_\banach),
    \end{align}
    which shows the claim in \eqref{eq:frechetderiv}.

    The mapping $f_2$ is linear and therefore coincides with its Fréchet derivative:
    \begin{align}
    (D f_2 (B))(h) = f_2(h) = P h(\cdot)\bar{\cdot}.
    \end{align}    
    Collecting the results, one can conclude from the chain rule for the Fréchet derivative:
    \begin{align}
        (D f_\id(B))(h)
        &=D{f_1}(f_2(B))( D{f_2}(B) (h))\\
        &= D\nu_\id(P B(\cdot)\bar{\cdot}) P h(\cdot) \bar{\cdot} \label{eq:chainrule}.
    \end{align}
    We return to calculating $M''$  using $M'(t)=f(M(t))$, see \eqref{eq:compact}: 
    \begin{align}
        M''(t)
        =  (f \circ M)'(t)
        = D f (M(t))\circ M'(t).\label{eq:Kdef}
    \end{align}
    We can further expand on this by using $M'(t)=f(M(t))$ again:
    \begin{align}
        \ldots &= D f(M(t)) \circ f(M(t)).
    \end{align}
    Equation \eqref{eq: identity value} and inserting the definition of $f$ yield for the multiplication with the identity operator $e: \banach \to \banach, e(h)=h$ that
    \begin{align}
        \ldots&= D (f_\id e)(M(t)) \circ f(M(t)).
    \end{align}
    Using the product rule for bilinear pointwise function multiplication in $\banach$ for the left term and defining the left- and right-multiplication operator as
    \begin{align}
    L:\banach \to \mathcal{L}(\banach,\banach), \, \left( L(g)\right)(h)=g\cdot h\\R:\banach \to \mathcal{L}(\banach,\banach), \, (R(g))(h)= h \cdot g ,
    \end{align}
    we get
    \begin{align}
        \ldots &=  (L(f_\id(M(t))) + R(M(t))(Df_\id(M(t))))\circ f(M(t)).
    \end{align}
    Substituting the chain rule calculated in \eqref{eq:chainrule}, we obtain
    \begin{align}
        \ldots = L(f_\id(M(t)))f(M(t)) + R(M(t))(D\nu_\id(P M(\cdot)\bar{\cdot}) P f(M(t))(\cdot) \bar{\cdot})
    \end{align}
    which is as a combination of continuous functions continuous and, accordingly, ensures 
    $M \in \mathcal{C}^2([0,1],\banach)$.
    Additionally, as $[0,1]$ is compact, the range $M''([0,1])$ is compact and therefore bounded.
    
    \mypar{Euler discretization}
    For some fixed end time $T>0$, consider an equidistant time discretization $\{t_0=0 < t_1 <\ldots < t_n=T\}$ into $n$ intervals of length $\delta_t:=\frac{T}{n}$ and the associated forward-Euler discretization of the embedded flow equation \eqref{eq:compact}
    \begin{align}
        M_\euler (t_{k+1}) &:= M_\euler(t_k) + f(M_\euler(t_k))(t_{k+1}-t_k) \nonumber\\
        M_\euler (t_0) &:= \id
        \label{eq:eulerscheme}.
    \end{align}
    Note that due to the embedding, each $M_\euler(t_k)$ maps from $\Omega$ into $\R^{4\times 4}$.

    Our goal is to show convergence of $M_\euler(t_n)$ to $M(t_n)$, following the proof for the scalar-valued case in \cite[Sect.~2.2, Thm.~2.4]{atkinson2009numerical}. This requires an estimate of the truncation error. For this, we first use the fundamental theorem of calculus for Bochner integrals \cite{liu2015stochastic} twice and rearrange:
    \begin{align}
    M(t_{k+1})&= M(t_k)+ \int_{t_k}^{t_{k+1}} M'(s) ds\\
    &= M(t_k)+ \int_{t_k}^{t_{k+1}} \left(M'(t_k)+ \int_{t_k}^{s} M''(t) dt \right) ds\\
    &= M(t_k)+(t_{k+1}-t_k) M'(t_k) + \int_{t_k}^{t_{k+1}}\int_{t_k}^{s} M''(t)\,dt\,ds.
    \end{align}
    This is equivalent to
    \begin{align}
    M(t_{k+1})-M(t_k)-(t_{k+1}-t_k) M'(t_k)= \int_{t_k}^{t_{k+1}}\left(\int_{t_k}^{s} M''(t) dt\right) ds.
    \end{align}
    Taking the norm on both sides and using the Bochner inequality~\cite{liu2015stochastic} (recall that $\|\cdot\|_\banach$ is the supremum norm as defined in \eqref{eq:bnormsup}), we obtain
    \begin{align}
    \|M(t_{k+1})-M(t_k)-(t_{k+1}-t_k) M'(t_k)\|_\banach &\leq \int_{t_k}^{t_{k+1}}\int_{t_k}^{s} \|M''(t)\|_\banach\,dt\,ds \\
    & \leq \frac{(t_{k+1}-t_k)^2}{2} \max_{t \in[t_k,t_{k+1}]} \|M''(t)\|_\banach.
    \end{align}
    Noting that $M'(t_k)=f(M(t_k))$, this provides a bound for the one-step truncation error of the forward Euler method depending on the norm of the second derivatives $M''$.

    This extends the first step of the convergence proof of the forward Euler method for the scalar-valued case in \cite[Sect.~2.2, Thm.~2.4]{atkinson2009numerical} to the Banach space-valued case; the remainder of the proof is identical if one replaces absolute values with $\|\cdot\|_\banach$. This ensures
    \begin{align}
    \|M(t)-M_\euler(t)\|_\banach \leq \frac{e^{t L_\nu}-1}{L_\nu} \frac{\delta_t}{2} \max_{s \in[0,t]} \|M''(s)\|_\banach,\label{eq:mtmeubound}
    \end{align}
    where $L'_\nu$ is the Lipschitz constant of $f$. Lipschitz continuity of $f$ on some open neighborhood of the image of $M$ was shown in Eq. \eqref{eq: lipschitz} during the proof of Thm.~\ref{thm:existence}. It requires Lipschitz continuity of $\nu_\id$, which is ensured here by the $C^1$ assumption on $\nu_\id$ and compactness of $\Omega$.

 Overall, \eqref{eq:mtmeubound} proves convergence of the (final-time) solution of the forward Euler approximation $M_\euler(T)$ to the solution $M(T)$ of the time-continuous flow equation, i.e.,
    \begin{align}
    M_\euler(T) \to M(T) \quad \text{in $O(\delta_t)$}. \label{eq:eu_convergence}
    \end{align}

\mypar{Exponential discretization}
    For the second step, in order to motivate the construction of the exponential discretization scheme, consider the matrix flow equation~\eqref{eq:matflow2-main} with spatially constant velocity $\mu \in \mathfrak{g}$, starting at $g_0 \in G$:
    \begin{align}
    \frac{\partial M}{\partial t}(x,t) &= \mu_{M(x,t)} \label{eq:1} \\
    M(\cdot, 0) &= g_0. 
    \end{align}
    Due to the right invariance of $\mu$, the solution takes the analytic form (compare~\eqref{eq: integral curve}): 
    \begin{align}
    M(x,t)=\exp(t\mu_{\id})g_0 ,\label{eq: analytic}
    \end{align}
    where $\exp$ denotes the matrix exponential. 
    Applying this idea for each time interval $[t_k, t_{k+1}]$ and to each point $x$ separately, and setting $\mu := \nu(PM(t_k)\bar{x})$ and $g_0:=M(x,t_k)$, we obtain the time-discrete exponential scheme
    \begin{align}
        M_{\expon}(t_{k+1}) &:=  \exp((t_{k+1}-t_k)\nu_\id(P M_{\expon}(t_k)\bar{\cdot})) \, M_{\expon}(t_k) \nonumber\\*
        M_{\expon}(t_{0}) &:= id
        \label{eq:expscheme}.
    \end{align}
    Fixing $x\in \Omega$ and rewriting the exponential scheme evaluated at the endpoint $T$ as a product of matrices, results in
    \begin{align}      
    M_{\expon}(x,T)&=\rprod_{k=0}^{n-1}\exp\left(\delta_t \nu_{\id}\left(P M_{\expon}\left(x,t_k\right) \bar{x}\right)\right), \label{eq: approx}
    \end{align}
    where $\rprod$ denotes the matrix product evaluated from right to left, i.e.,
    \begin{align}
    \rprod_{k=0}^{n-1} A_k := A_{n-1} A_{n-2} \cdots A_1 A_0.
    \end{align}

    We will later need a bound on the norm of $M_\expon$. We use the fact that $|\exp(M)|\leq e^{|M|}$ to derive
    \begin{align}
        |M_{\expon}(x,t_k)| &\leq \prod_{j=0}^{k-1}|\exp\left(\delta_t \nu_{\id}\left(P M_{\expon}\left(x, t_j\right)\bar{x} \right)\right) | \\
        &\leq \prod_{j=0}^{k-1} e^{|\delta_t \nu_{\id}\left(P M_{\expon}\left(x,t_j\right)\bar{x} \right)|} \\
        &= e^{\sum_{j=0}^{k-1} |\delta_t\nu_{\id}\left(P M_{\expon}\left(x,t_j\right)\bar{x} \right)|}.
    \end{align}
    As $\nu_{\id} \in C^1(\Omega, \R^{4\times4})$ is bounded on the compact set $\Omega$ by some $C_{\nu}>0$, we continue
    \begin{align}
        |M_{\expon}(x,t_k)| &\leq e^{\sum_{j=0}^{k-1} \left(\delta_t C_{\nu}\right)}\\
        &= e^{t_k C_{\nu}}\\
                &\leq e^{T C_{\nu}}=: C'_{\nu,T}. \label{eq:expbound}
    \end{align}
    Importantly, this bound depends neither on the time step $\delta_t$ nor on the spatial position~$x$.

    In order to bound the difference between the forward Euler iteration in the previous section and the exponential approach defined by \eqref{eq:expscheme}, we define the error
    \begin{align}
    e_{k}(x) := M_\euler(x,t_k) - M_\expon(x,t_k). \label{eq:ekdef}
    \end{align}
    The proof for bounding the error closely follows the strategy for proving convergence of the classical forward Euler method~\cite{atkinson2009numerical}.    By definition of $M_\euler$ and $M_\expon$,
    \begin{align}
         |e_{k+1}(x)| &= |M_{\euler}(x,t_k) + \delta_t \nu(P M_{\euler}(x,t_k)\Bar{x})_{M_{\euler}(x,t)}  \nonumber\\
         &\quad -\exp\left(\delta_t\nu_{\id}\left(P M_{\expon}\left(x,t_k\right)\Bar{x} \right)\right) M_{\expon}(x,t_k)|.
    \end{align}
    We rewrite the matrix exponential by its Taylor series:
    \begin{align}
         |e_{k+1}(x)| & \leq \bigg| M_{\euler}(x,t_k) + \delta_t \nu_{\id}(P M_{\euler}(x,t_k)\Bar{x}) \, M_{\euler}(x,t_k) \nonumber\\
         &\quad - \sum_{l=0}^\infty \frac{(\delta_t)^l\nu_{\id}\left(P M_{\expon}\left(x,t_k\right)\Bar{x} \right)^l}{l!} M_{\expon}(x,t_k) \bigg| . 
    \end{align}
    Reordering the terms and using the triangle inequality, we obtain 
    \begin{align}
         |e_{k+1}(x)| &= |M_{\euler}(x,t_k)- M_{\expon}(x,t_k)| \nonumber \\
         &\quad+ \delta_t
         | 
         \nu_{\id}(P M_{\euler} (x,t_k)\Bar{x})M_{\euler}(x,t_k)- \nu_{\id}(P M_{\expon}(x,t_k)\Bar{x})M_{\expon}(x,t_k) |  \nonumber \\
         &\quad+ \left|\sum_{l=2}^\infty \frac{(\delta_t)^l\nu_\id\left(P M_{\expon}\left(x,t_k\right)\Bar{x} \right)^l}{l!} M_{\expon}(x,t_k)\right| \\
         &= |e_k(x)| \nonumber \\
         &\quad+ \delta_t |
         f(M_\euler(t_k))(x) - f(M_\expon(t_k))(x) | \nonumber \\
         &\quad+ \left|\sum_{l=2}^\infty \frac{(\delta_t)^l\nu_\id\left(P M_{\expon}\left(x,t_k\right)\Bar{x} \right)^l}{l!} M_{\expon}(x,t_k)\right|. \label{eq:someboundek}
    \end{align}
We rewrite the second term in the sum using the Lipschitz estimate \eqref{eq: lipschitz}:
\begin{align}
         & \delta_t |
         f(M_\expon(t_k))(x) - f(M_\euler(t_k))(x) | \\
\leq & \delta_t (L_{\nu} (\sup_{x \in \Omega} |\bar{x}| \|M_\expon(t_k)\|_\banach + C_{\nu})\|M_\euler(t_k)-M_\expon(t_k)\|_{\banach} \\
\overset{\eqref{eq:expbound}}{\leq} & \delta_t (L_{\nu} (\sup_{x \in \Omega} |\bar{x}| C'_{\nu,T} + C_{\nu}) \|e_k\|_{\banach} \\
=& \delta_t C''_{\nu,T} \|e_k\|_{\banach}
\end{align}
for some constant $C''_{\nu,T}$.

We continue bounding the third term in \eqref{eq:someboundek}:
\begin{align}
    &\quad \left| \sum_{l=2}^\infty \frac{(\delta_t)^l\nu_\id\left(P M_{\expon}\left(x,t_k\right)\Bar{x} \right)^l}{l!} M_{\expon}(x,t_k) \right|\\
    &=(\delta_t)^2\left|\nu_\id \left(P M_{\expon}\left(x,t_k\right)\Bar{x} \right)^2 \sum_{l=0}^\infty \frac{(\delta_t)^l\nu_\id\left(P M_{\expon}\left(x,t_k\right)\Bar{x} \right)^l}{(l+2)!} M_{\expon}(x,t_k)\right|\\
    &\leq (\delta_t)^2\left|\nu_\id \left(P M_{\expon}\left(x,t_k\right)\Bar{x} \right)\right|^2 \sum_{l=0}^\infty \frac{ \left|(\delta_t)^l\nu_\id\left(P M_{\expon}\left(x,t_k\right)\Bar{x} \right)^l \right|}{l!}  |M_{\expon}(x,t_k)|\\
    &\overset{\eqref{eq:expbound}}{\leq} (\delta_t)^2 (C_\nu)^2 e^{\delta_t C_\nu}C'_{\nu,T}\\
    &=: (\delta_t)^2C'''_{\nu,T}.
\end{align}

Inserting both bounds into \eqref{eq:someboundek} leads to
\begin{align}
|e_{k+1}(x)| \leq |e_k(x)| + \delta_t C''_{\nu,T} \|e_k\|_{\banach} +(\delta_t)^2C'''_{\nu,T}. \label{eq:finalboundek}
\end{align}
Taking the supremum on both sides yields:
\begin{align}
    \|e_{k+1}\|_\banach \leq (1 + \delta_t C''_{\nu,T}) \|e_k\|_\banach +(\delta_t)^2C'''_{\nu,T}. \label{eq: onestep}
\end{align}
    Iterating this estimate leads to an estimate for the final error $e_n$:
    \begin{align}
        \|e_n\|_\banach \leq &(1 + \delta_t C''_{\nu,T})^n\|e_0\|_\banach \nonumber\\
        &+ (1+ (1 + \delta_t C''_{\nu,T}) + (1 + \delta_t C''_{\nu,T})^2+\hdots+ (1 + \delta_t C''_{\nu,T})^{n-1})(\delta_t)^2C'''_{\nu,T}.
    \end{align}
    By the \emph{finite} geometric series,
    \begin{align}
        \|e_n\|_\banach &\leq (1 + \delta_t C''_{\nu,T})^n\|e_0\|_\banach - \frac{1-(1 + \delta_t C''_{\nu,T})^n}{\delta_t C''_{\nu,T}}(\delta_t)^2 C'''_{\nu,T}\\
        &= (1 + \delta_t C''_{\nu,T})^n\|e_0\|_\banach + \frac{(1 + \delta_t C''_{\nu,T})^n-1}{C''_{\nu,T}} \delta_t C'''_{\nu,T}\\
        &\leq (1 + \delta_t C''_{\nu,T})^n\|e_0\|_\banach + \frac{e^{n\delta_t C''_{\nu,T}}-1}{C''_{\nu,T}}\delta_t C'''_{\nu,T}\\
        &= (1 + \delta_t C''_{\nu,T})^n\|e_0\|_\banach + \frac{e^{T C''_{\nu,T}}-1}{C''_{\nu,T}}\delta_t C'''_{\nu,T}.\label{eq:enestimatefinal}
    \end{align}
    As the evolution of the matrix fields as well as of both schemes starts in the identity, the initial error $\|e_0\|_\banach$ is 0.
    Thus,
    \begin{equation}
        \| M_\expon (T) - M_\euler (T)\|_\banach = \|e_n\|_\banach \leq \delta_t C^{(4)}_{\nu,T},
        \label{eq:exeu_convergence}
    \end{equation}
    showing that the solution of the exponential discretization converges to the solution of the Euler scheme for $\delta_t \to 0$.
    
     Combining this with the convergence of the Euler scheme \eqref{eq:eu_convergence} gives 
    \begin{align}
        M_\expon(T) \to M(T) \quad \text{in $O(\delta_t)$} \label{eq:ex_convergence}
    \end{align}
    for arbitrary $T>0$.
    Thus, the solution of the exponential discretization converges to the continuous solution.

    \mypar{Decomposition property}
    Now, we can infer the decomposition property by defining the difference
    between the exponential scheme and the analytical solution as
    \begin{align}
        e_{\delta_t,T,\nu_\id}(x) := M(x,T) - M_\expon (x,T).
    \end{align}
    and calculate
    \begin{align}
    M(x,2T) = M_\expon(x,2T)+e_{\delta_t,2T,\nu_\id}(x).
    \end{align}
    Inserting the definition of $M_{\expon}$ in \eqref{eq: approx}, we obtain
    \begin{align}
    ...&=\rprod_{k=0}^{2n-1}\exp\left(\delta_t\nu_\id\left(P M_{\expon}\left(x,t_k\right)\Bar{x}\right)\right)+e_{\delta_t,2T,\nu_\id}(x)\\
    &=\rprod_{k=n}^{2n-1}\exp\left(\delta_t\nu_\id\left(P M_{\expon}\left(x,t_k\right)\Bar{x}\right)\right) \rprod_{k=0}^{n-1}\exp\left(\delta_t\nu_\id\left(P M_{\expon}\left(x,t_k\right)\Bar{x}\right)\right)+e_{\delta_t,2T,\nu_\id}(x)\\
    &\overset{(*)}{=} M_\expon(PM_\expon(x,T)\bar{x},T) M_\expon(x,T)+e_{\delta_t,2T,\nu_\id}(x)\\
    &=\left(M(PM_\expon(x,T)\bar{x},T) - e_{\delta_t,T,\nu_{\id}}(PM_\expon(x,T)\bar{x})\right) \left(M(x,T) - e_{\delta_t,T,\nu_\id}(x)\right) \nonumber\\
    &\quad + e_{\delta_t,2T,\nu_\id}(x) \\
    &= M\left(PM_\expon(x, T)\Bar{x}, T\right)M\left(x, T\right)
    - e_{\delta_t,T,\nu_\id}(PM_\expon(x,T)\bar{x}) M(x, T) \nonumber\\
    &\quad - M\left(P M_\expon(x, T)\Bar{x}, T\right)e_{\delta_t,T,\nu_\id}(x)
    + e_{\delta_t,T,\nu_\id} (P M_\expon(x, T)\Bar{x}) e_{\delta_t,T,\nu_\id}(x) \nonumber\\
    &\quad + e_{\delta_t,2T,\nu_\id}(x).
    \end{align}
    The step in $(\ast)$ will be further explained below. As this equation for $M(x,2T)$ holds for all $\delta_t>0$, the error terms vanish for $\delta_t\to 0$. As further $M_\expon(T)$ is bounded for arbitrarily small $\delta_t > 0$, we obtain the claimed decomposition property:
    \begin{align}
     M(x,2T)
     &= \lim_{\delta_t \to 0} M\left(PM_\expon(x, T)\Bar{x}, T\right)M\left(x, T\right)\\
     &= M\left(PM(x, T)\Bar{x}, T\right)M\left(x, T\right),
    \end{align}
    where the last equality follows from $\lim_{\delta_t\to 0} M_\expon(x,T)= M(x,T)$ and the continuity of $M$. This concludes the proof of the theorem.

    $(\ast)$: We further explain the steps in this equality.
    The equality 
    \begin{align}
         M_\expon(x,T)=\rprod_{k=0}^{n-1}\exp\left(\delta_t\nu_\id\left(P M_{\expon}\left(x,t_k\right)\Bar{x}\right)\right)
    \end{align}
    holds by definition of the exponential scheme, see \eqref{eq: approx}.
    We now show the equality
    \begin{align}
         M_\expon(PM_\expon(x,T)\barx,T)=\rprod_{k=n}^{2n-1}\exp\left(\delta_t\nu_\id\left(P M_{\expon}\left(x,t_k\right)\Bar{x}\right)\right).
         \label{eq:prod-claim}
    \end{align}
    From \eqref{eq: approx} with $x=P M_\expon(x,T)\barx$, it follows
    \begin{align}
        M_\expon(PM_\expon(x,T)\barx,T)&= \rprod_{k=0}^{n-1}\exp\left(\delta_t \nu_{\id}\left(P M_{\expon}\left(PM_\expon(x,T)\barx,t_k\right) \overline{PM_\expon(x,T) \barx}\right)\right) . \label{eq:prod-proof0}
    \end{align}
    The projection operator $P$ and mapping $\bar{\cdot}$ to homogeneous coordinates cancel:
    \begin{align}
         \ldots &= \rprod_{k=0}^{n-1}\exp\left(\delta_t \nu_{\id}\left(P M_{\expon}\left(PM_\expon(x,T)\barx,t_k\right) M_\expon(x,T) \barx \right)\right).   
    \end{align}
    Splitting the right-most factor from the matrix product, we obtain
    \begin{align}
        \ldots = &\rprod_{k=1}^{n-1}\exp\left(\delta_t \nu_{\id}\left(P M_{\expon}\left(PM_\expon(x,T)\barx,t_k\right) M_\expon(x,T) \barx \right)\right)  \nonumber \\ 
        &\cdot \exp\left(\delta_t \nu_{\id}\left(P M_{\expon}\left(PM_\expon(x,T)\barx,t_0\right) M_\expon(x,T) \barx \right)\right).
    \end{align}
    From the definition of the exponential scheme, it holds that $M_\expon(\cdot,t_0)=id$, see \eqref{eq:expscheme}, and the expression simplifies to 
    \begin{align}
        M_\expon(PM_\expon(x,T)\barx,T) = &\rprod_{k=1}^{n-1}\exp\left(\delta_t \nu_{\id}\left(P M_{\expon}\left(PM_\expon(x,T)\barx,t_k\right) M_\expon(x,T) \barx \right)\right)  \nonumber \\ 
        &\cdot \exp\left(\delta_t \nu_{\id}\left(P M_\expon(x,T) \barx \right)\right).  
    \end{align}
    As $t_n=T$, it holds
    \begin{align}
        \exp\left(\delta_t \nu_{\id}\left(P M_\expon(x,T) \barx \right)\right) = \exp\left(\delta_t \nu_{\id}\left(P M_\expon(x,t_n) \barx \right)\right) , \label{eq:prod-proof}
    \end{align}
    which is precisely the right-most factor in \eqref{eq:prod-claim}.
    This shows that the right-most factors in the products of \eqref{eq:prod-claim} and \eqref{eq:prod-proof0} coincide. Iteratively applying the fact that each factor in the exponential scheme \eqref{eq:expscheme} depends only on the (temporally) previous one, we arrive at the claimed equality.
\end{proof}


\bibliography{sn-bibliography}
\newpage
\begin{appendices}
\section{Hyperparameter Tuning Results}\label{sec:hyperparameter-tuning}
\begin{table}[tbh]
\centering
{\begin{tabular}{llllllllll} %
    \toprule
     & \multicolumn{3}{c}{vector field fitting} & \multicolumn{3}{c}{synthetic deformations} & \multicolumn{3}{c}{inter-patient registration}\\
     method & $w_0$ & lr & sf & $w_0$ & lr & sf & $w_0$ & lr & sf\\
     \hline
     SVF \cite{han2023diffeomorphic}&  4.85  & 0.017 & 0.009 & 1.1e-4 &  3.8e-5 & 0.053 &  15.18 & 0.004 & 0.006\\
     SE(3) (ours) & 4.12 & 0.007 & 0.006 & 2.8e-3 &  0.001 & 0.058 & 13.11 & 0.005 & 0.002 \\
     SIM(3) (ours) & 4.89  & 0.009 & 0.017 & 2.8e-3 & 0.001 & 0.058 & 13.92 & 0.008 & 0.018 \\
     \bottomrule
     \end{tabular}
}
    \caption{Hyperparameters after tuning with the Optuna framework: SIREN initial frequency scaling ($w_0$), learning rate (lr), and scaling factor (sf); see Section~\ref{sec:network-architecture}.}
    \label{tab:hyperpar}
\end{table}
We used the Optuna framework \cite{optuna_2019} for choosing the hyperparameters for the experiments in Section \ref{sec4}. The results of the parameter tuning are presented in Table~\ref{tab:hyperpar}. Additional information is presented in figures \ref{fig:parametertuning_Deformationfit}, \ref{fig:parametertuning}, \ref{fig:parametertuningOASIS}. The bar graphs on the left show Optuna's estimate of how strong the specific hyperparameter impacts the score on the set of test image pairs. On the right are 2D projections of the score functions with respect to different pairs of hyperparameters.

\begin{figure}[tbp]
    \centering
    \subfloat[][SVF]
    {\includegraphics[width=0.45\textwidth]{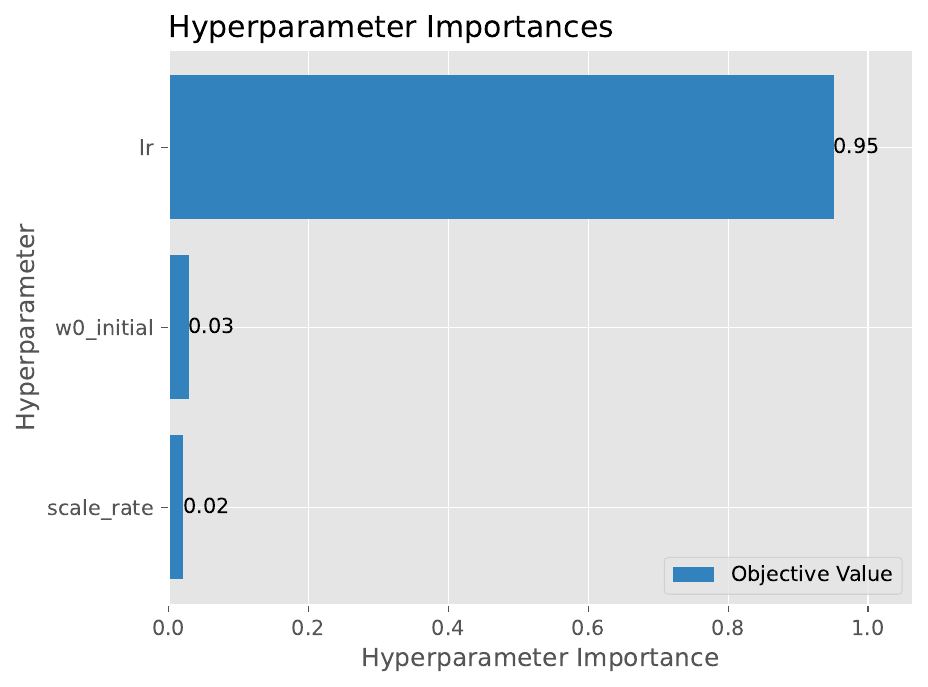}}\hspace{0.1\textwidth}%
    \includegraphics[width=0.45\textwidth]{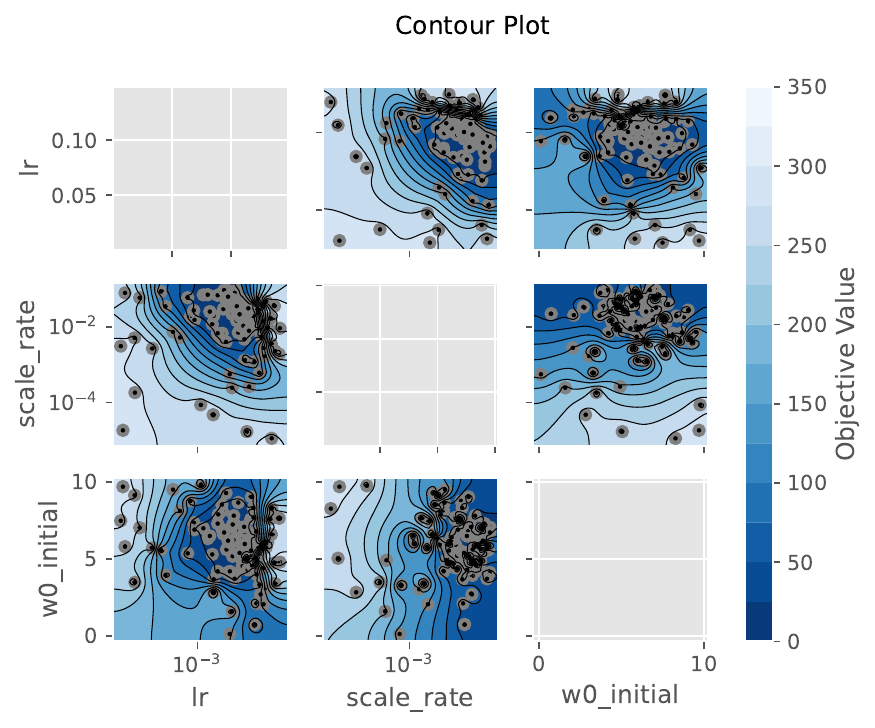}\par
    \subfloat[][$\SE(3)$]{\includegraphics[width=0.45\textwidth]{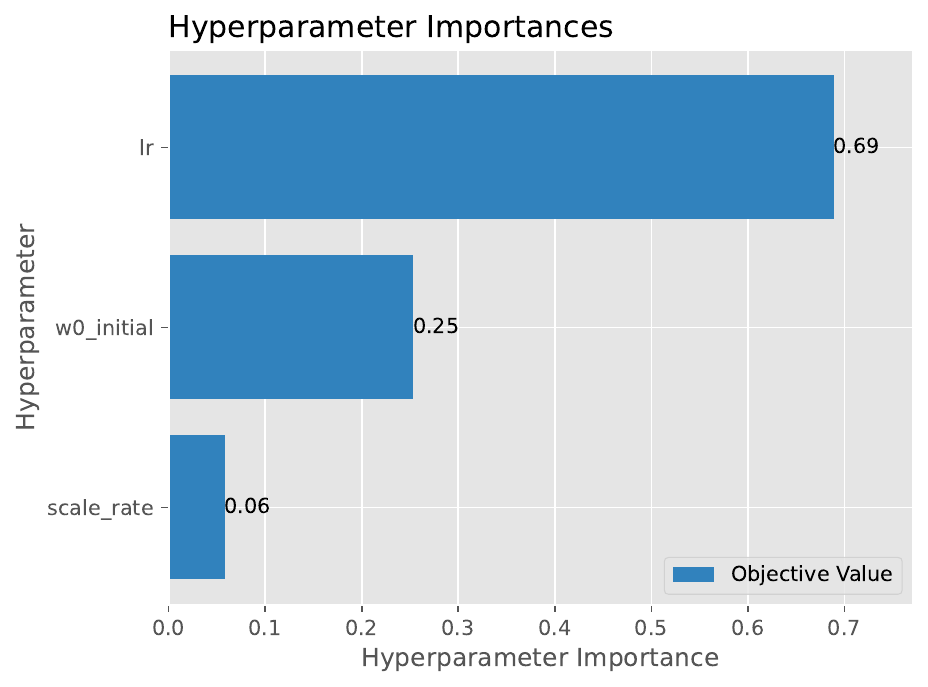}}\hspace{0.1\textwidth}%
    \includegraphics[width=0.45\textwidth]{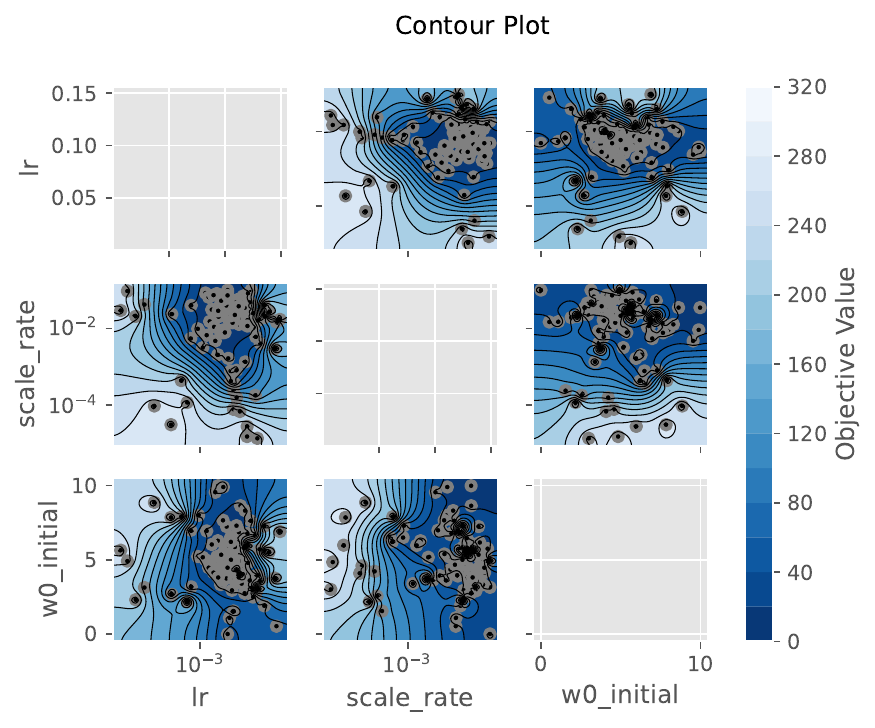}\\
    \subfloat[][$\SIM(3)$]
    {\includegraphics[width=0.45\textwidth]{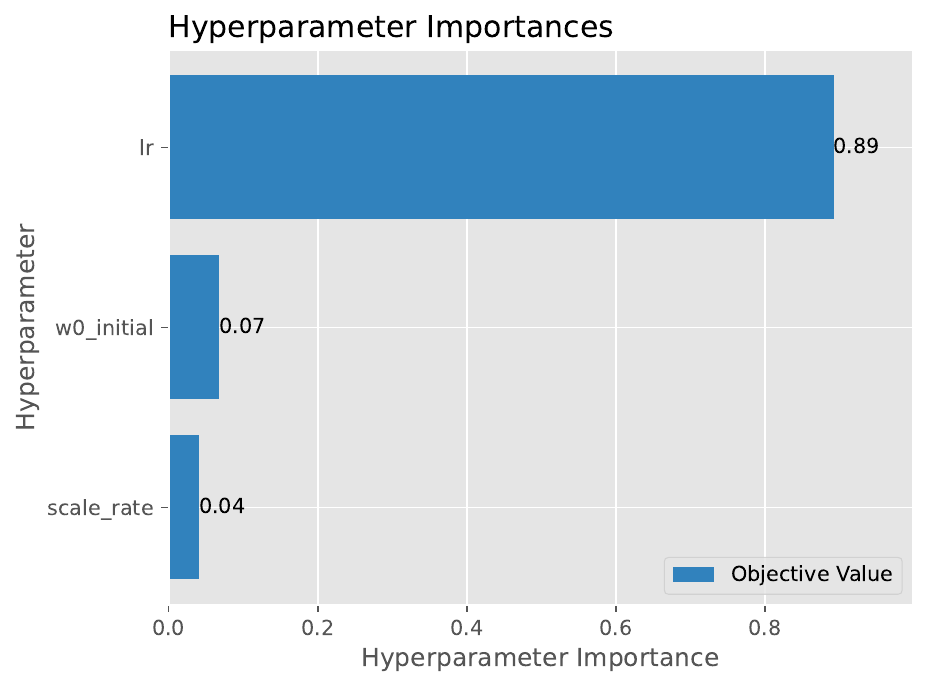}}\hspace{0.1\textwidth}%
    \includegraphics[width=0.45\textwidth]{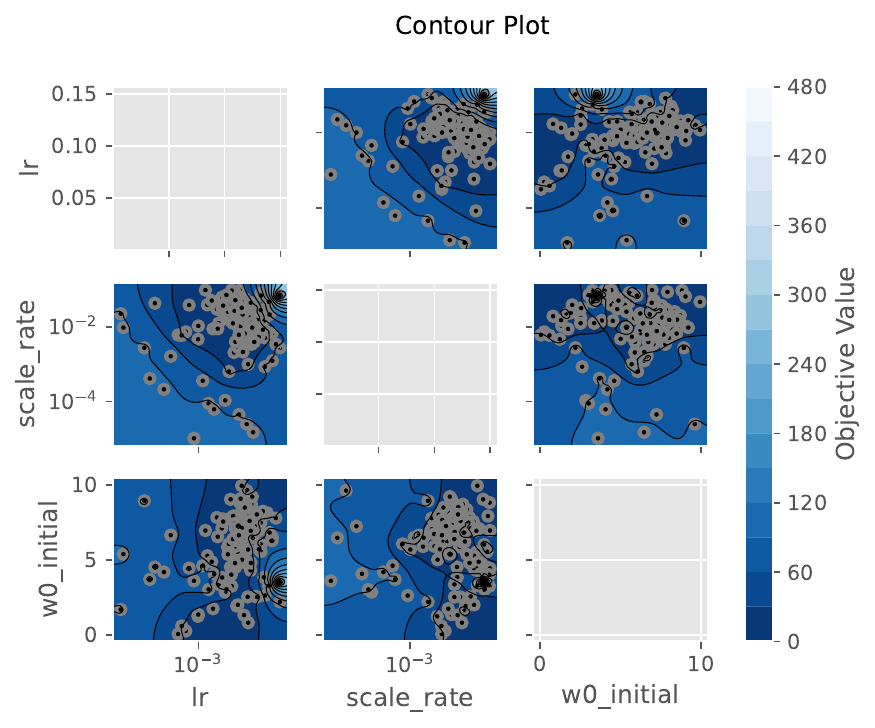}\\  
     \caption{Hyperparameter for fitting 5 synthetically generated deformation fields (Fig.~\ref{fig:deformation-fitting}). Parameters were optimized to minimize the RMSE between the given and generated deformation fields.}
    \label{fig:parametertuning_Deformationfit}
\end{figure}

\begin{figure}[tbp]
    \centering
    \subfloat[][SVF]
    {\includegraphics[width=0.45\textwidth]{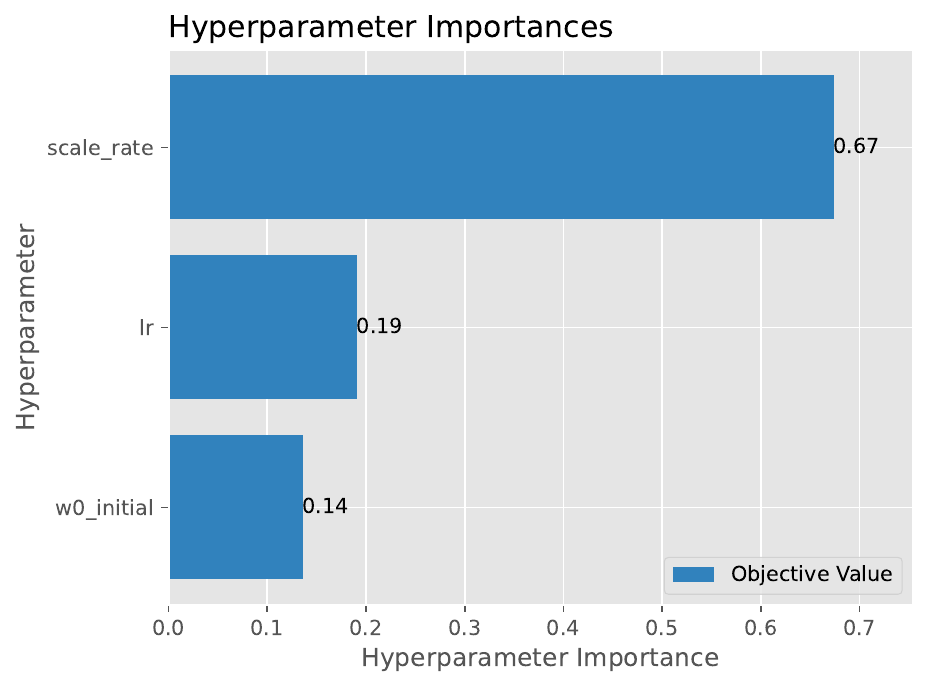}}\hspace{0.1\textwidth}%
    \includegraphics[width=0.45\textwidth]{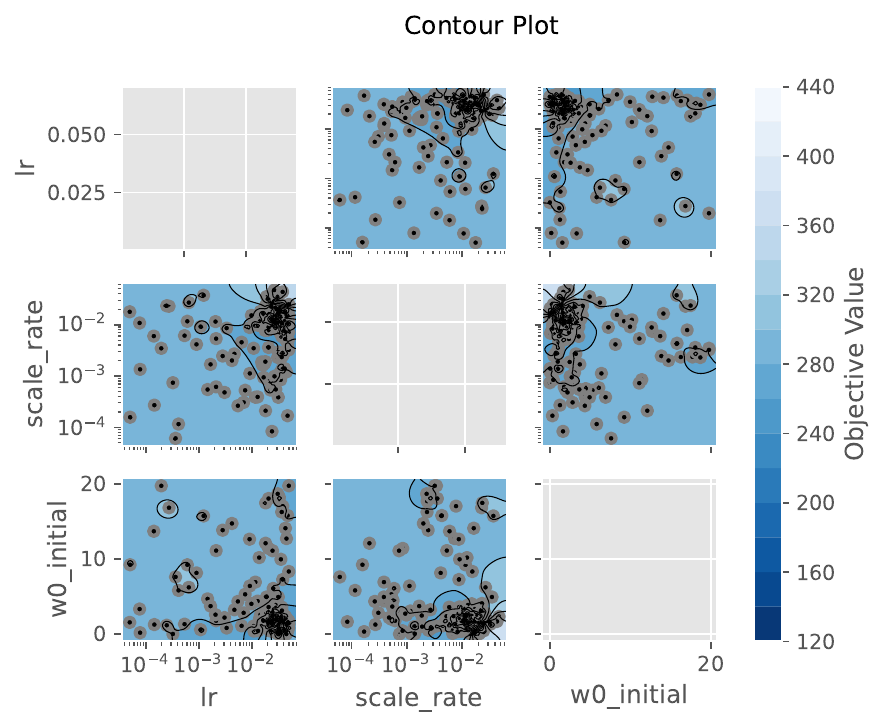}\par
    \subfloat[][SE3]{\includegraphics[width=0.45\textwidth]{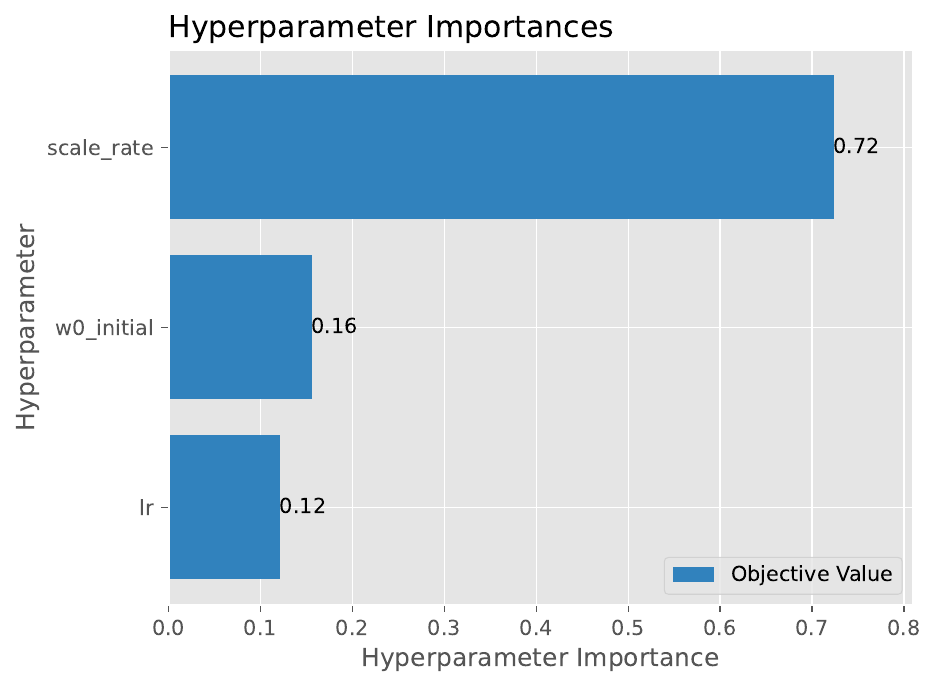}}\hspace{0.1\textwidth}%
    \includegraphics[width=0.45\textwidth]{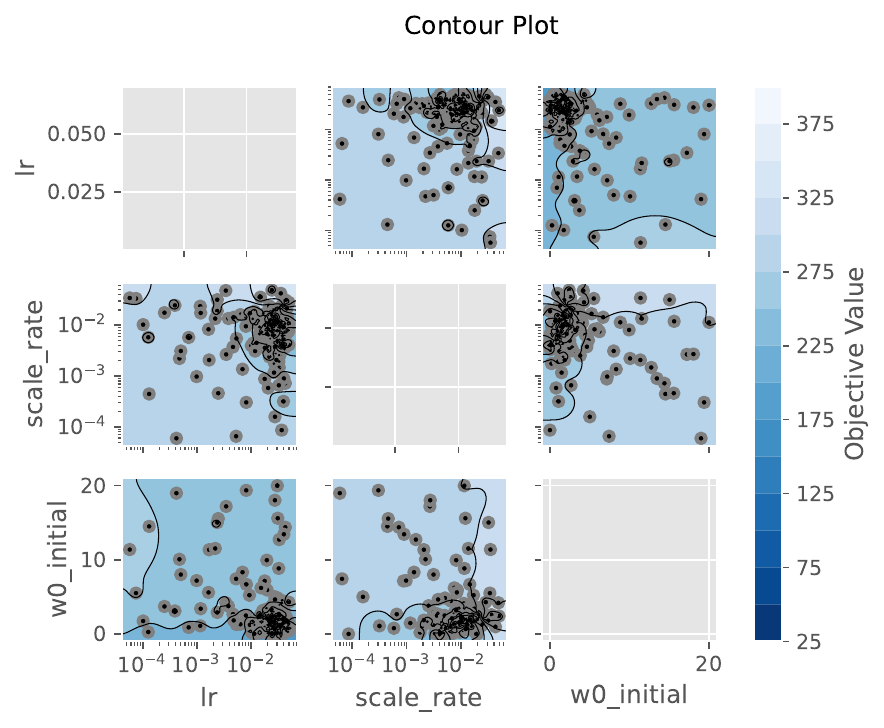}\\
    \subfloat[][SIM3]
    {\includegraphics[width=0.45\textwidth]{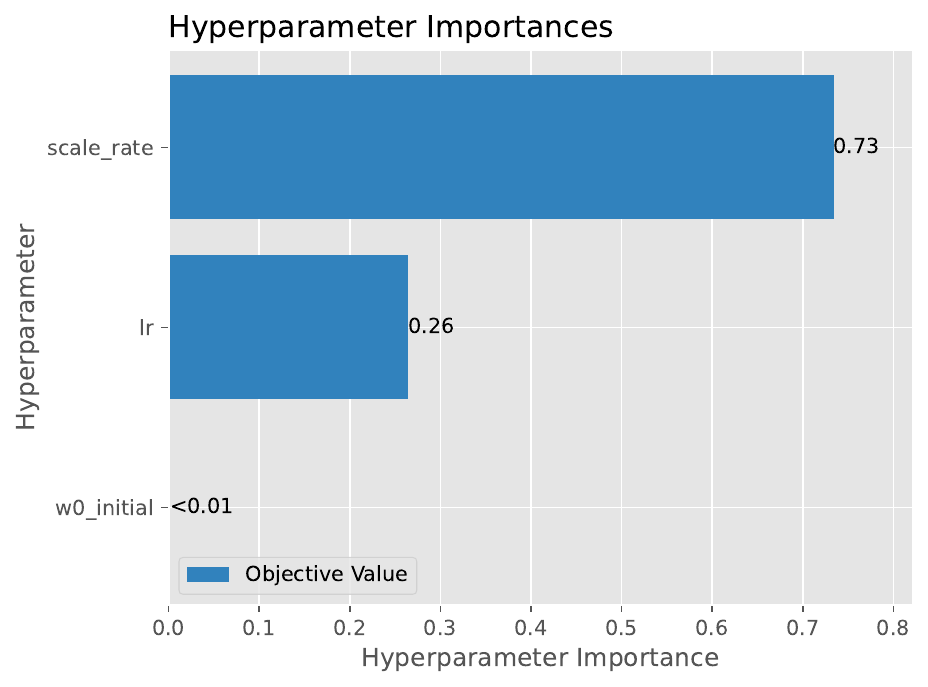}}\hspace{0.1\textwidth}%
    \includegraphics[width=0.45\textwidth]{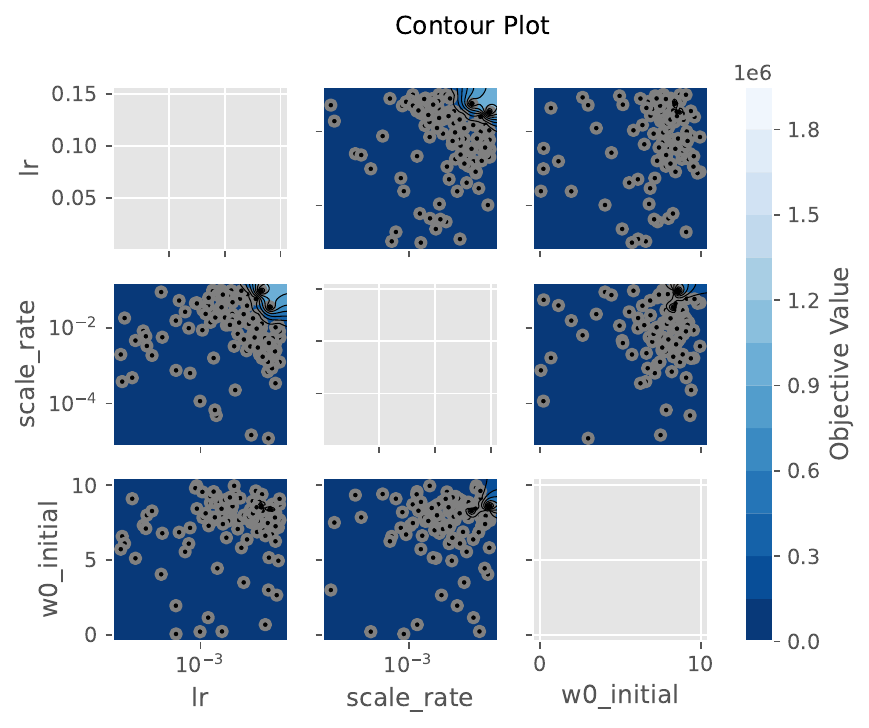}\\    
    \caption{Hyperparameter Tuning for registration with synthetic deformations (Fig.~\ref{fig:self-warped}). For $\SIM(3)$, hyperparameters were ultimately chosen the same as for $\SE(3)$, as the automated search did not yield satisfactory results.}
    \label{fig:parametertuning}
\end{figure}

\begin{figure}[tbp]
    \centering
    \subfloat[][SVF]
    {\includegraphics[width=0.45\textwidth]{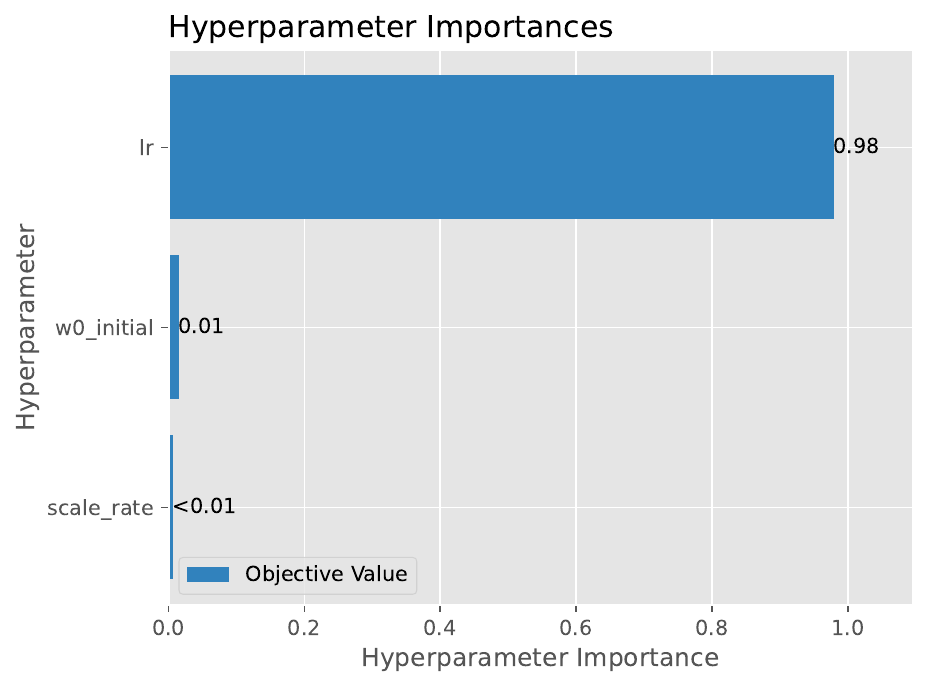}}\hspace{0.1\textwidth}%
    \includegraphics[width=0.45\textwidth]{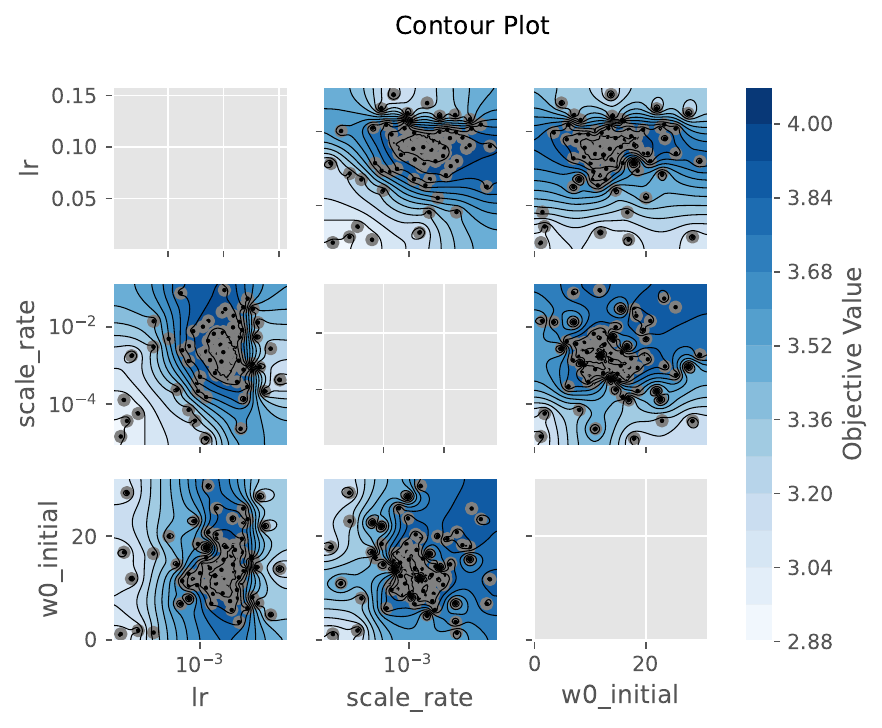}\par
    \subfloat[][$\SE(3)$]{\includegraphics[width=0.45\textwidth]{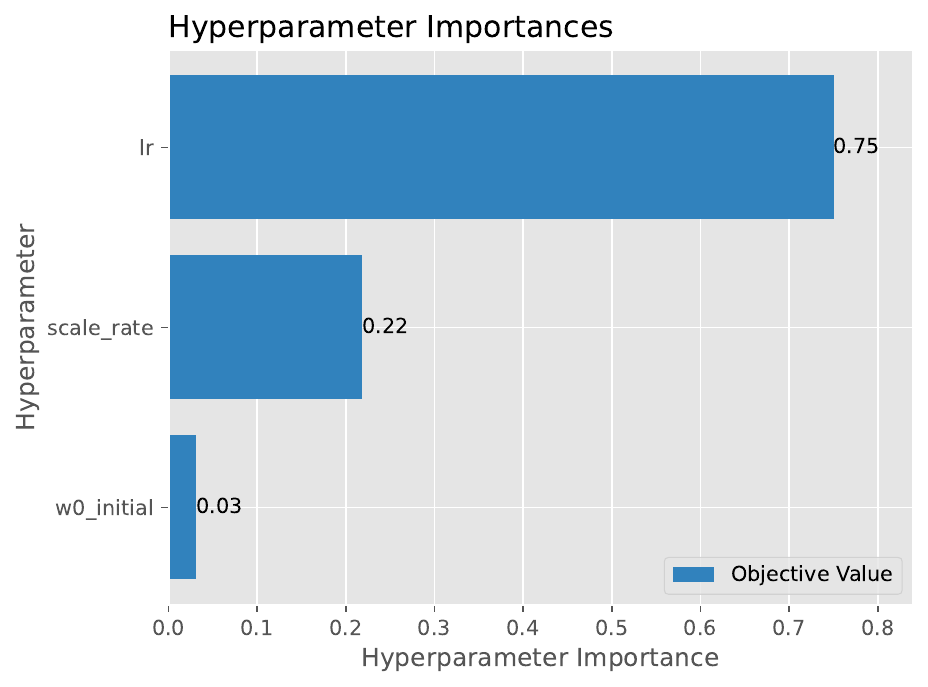}}\hspace{0.1\textwidth}%
    \includegraphics[width=0.45\textwidth]{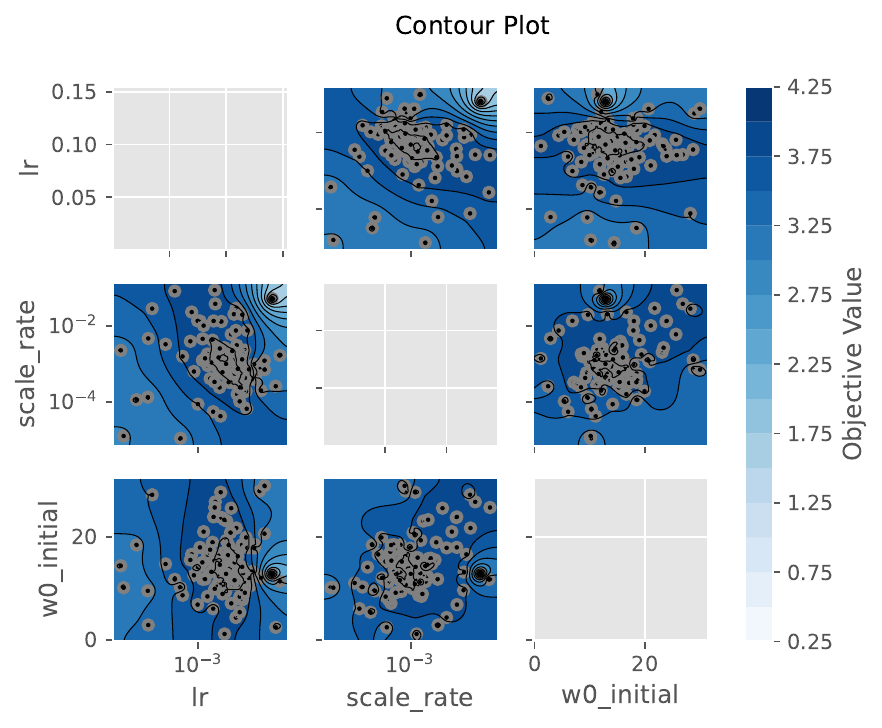}\\
    \subfloat[][$\SIM(3)$]
    {\includegraphics[width=0.45\textwidth]{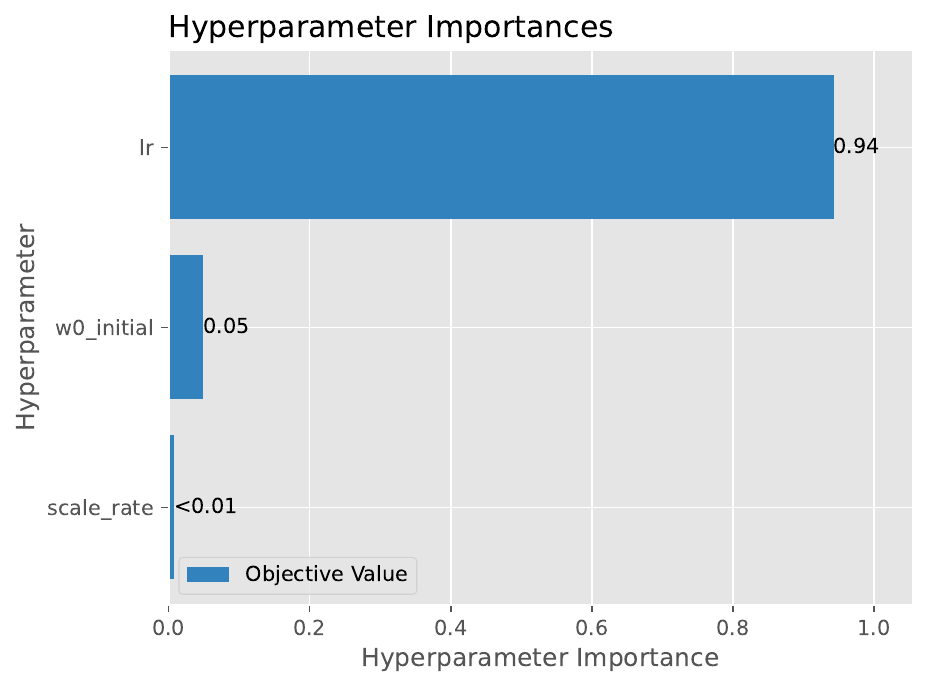}}\hspace{0.1\textwidth}%
    \includegraphics[width=0.45\textwidth]{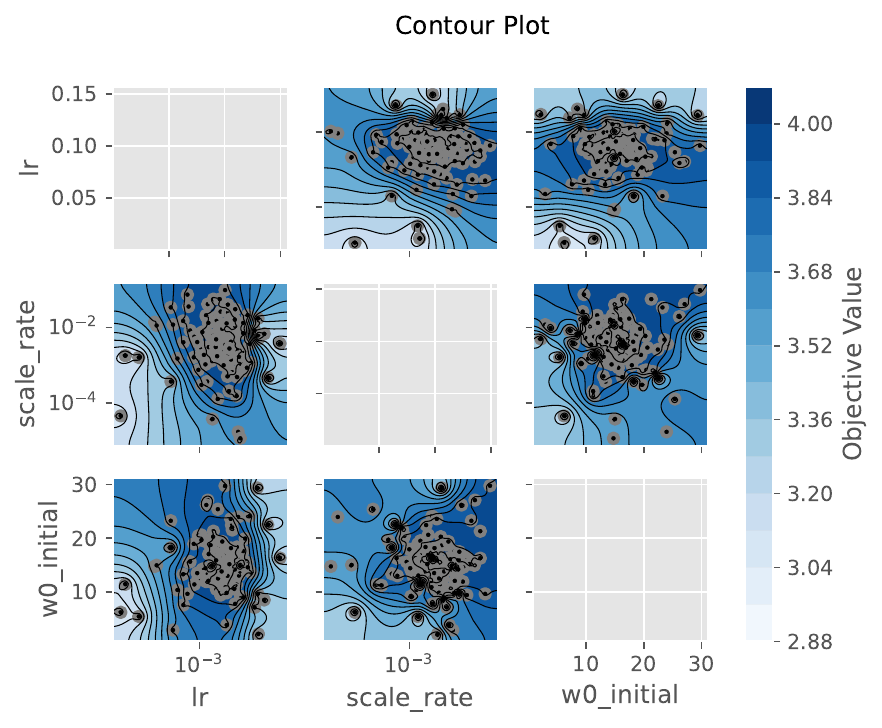}\\    
    \caption{Hyperparameter tuning on 5 selected scull stripped volume pairs from the OASIS dataset. Parameters were optimized to maximize the sum of the Dice scores (Table~\ref{tab:table}).}
    \label{fig:parametertuningOASIS}
\end{figure}




\end{appendices}

\end{document}